\documentclass{article} 
\usepackage{iclr2027_conference,times}

\usepackage[utf8]{inputenc} 
\usepackage[T1]{fontenc}    
\usepackage{url}            
\usepackage[ruled,vlined,linesnumbered]{algorithm2e}
\usepackage{amsmath}
\usepackage{amssymb}
\usepackage{url}
\usepackage{array}
\usepackage{microtype}
\usepackage{graphicx}
\usepackage{lineno}
\usepackage{subfigure}
\usepackage{color, colortbl}
\usepackage{xcolor}
\usepackage{colortbl}
\definecolor{mycitecolor}{HTML}{3498DC}
\definecolor{mylinkcolor}{HTML}{E74D3B}
\definecolor{myurlcolor}{HTML}{980000}
\definecolor{mydarkgreen}{HTML}{6a994e}
\definecolor{myorange}{HTML}{E7730D}
\definecolor{myblue}{HTML}{4594c1}
\definecolor{myyellow}{HTML}{FFDF00}
\definecolor{mypurple}{HTML}{C025CB}
\definecolor{RQframe}{HTML}{CFAB8D}
\definecolor{URAGColor}{HTML}{DE1A58}

\definecolor{sparse1}{HTML}{019E4F}
\definecolor{sparse2}{HTML}{E94127}
\definecolor{sparse3}{HTML}{017DC7}
\definecolor{dense1}{HTML}{3B6C73}
\definecolor{dense2}{HTML}{C1565E}
\definecolor{dense3}{HTML}{4682B4}

\definecolor{Appendixcolor}{HTML}{FFA500}  
\definecolor{Fgcolor}{HTML}{0000FF}
\definecolor{Sectioncolor}{HTML}{FFA500}
\definecolor{Subsectioncolor}{HTML}{FFA500}
\definecolor{Equationcolor}{HTML}{FFA500}
\definecolor{Definitioncolor}{HTML}{ea00ff}
\definecolor{Tablecolor}{HTML}{0000FF}
\definecolor{Observationcolor}{HTML}{FFA500}
\definecolor{Takeawaycolor}{HTML}{001BB7}
\definecolor{Interpretationcolor}{HTML}{ffb37c}
\definecolor{Algcolor}{HTML}{FFA500}
\definecolor{Promptcolor}{HTML}{FF6C0C}
\definecolor{RQcolor}{HTML}{001BB7}
\definecolor{Theoremcolor}{HTML}{ea00ff}
\definecolor{Lemmacolor}{HTML}{ea00ff}
\definecolor{Propositioncolor}{HTML}{de07fa}
\definecolor{Corollarycolor}{HTML}{FFA500}
\definecolor{Inequalitycolor}{HTML}{ff9400}
\definecolor{Assumptioncolor}{HTML}{E76F2E}

\newcommand{\Appendix}{\textcolor{Appendixcolor}{Appendix}}
\newcommand{\Figure}{\textcolor{Fgcolor}{Figure}}
\newcommand{\Section}{\textcolor{Sectioncolor}{Section}}

\newcommand{\Equation}{\textcolor{Equationcolor}{Equation}}
\newcommand{\Definition}{\textcolor{Definitioncolor}{Definition}}
\newcommand{\Table}{\textcolor{Tablecolor}{Table}}

\newcommand{\Theorem}{\textcolor{Theoremcolor}{Theorem}}
\newcommand{\Lemma}{\textcolor{Lemmacolor}{Lemma}}

\newcommand{\Assumption}{\textcolor{Assumptioncolor}{Assumption}}
\newcommand{\Corollary}{\textcolor{Corollarycolor}{Corollary}}

\usepackage[normalem]{ulem}
\usepackage{enumitem}
\usepackage{titletoc}

\newcommand{\Appendixref}[1]{\Appendix~{\hypersetup{linkcolor=Appendixcolor}\ref{#1}}}
\newcommand{\Figureref}[1]{\Figure~{\hypersetup{linkcolor=Fgcolor}\ref{#1}}}
\newcommand{\Sectionref}[1]{\Section~{\hypersetup{linkcolor=Sectioncolor}\ref{#1}}}

\newcommand{\Equationref}[1]{\Equation~{\hypersetup{linkcolor=Equationcolor}\ref{#1}}}
\newcommand{\Definitionref}[1]{\Definition~{\hypersetup{linkcolor=Definitioncolor}\ref{#1}}}
\newcommand{\Theoremref}[1]{\Theorem~{\hypersetup{linkcolor=Theoremcolor}\ref{#1}}}
\newcommand{\Lemmaref}[1]{\Lemma~{\hypersetup{linkcolor=Lemmacolor}\ref{#1}}}

\newcommand{\Corollaryref}[1]{\Corollary~{\hypersetup{linkcolor=Corollarycolor}\ref{#1}}}
\DeclareRobustCommand{\Tableref}[1]{%
    \Table~\hyperref[#1]{\textcolor{Tablecolor}{\ref*{#1}}}%
}

\newcommand{\Assumptionref}[1]{\Assumption~{\hypersetup{linkcolor=Assumptioncolor}\ref{#1}}}

\usepackage[most]{tcolorbox}

\newcounter{rq}

\newcounter{takeaway}
\renewcommand{\thetakeaway}{\arabic{takeaway}}

\newenvironment{takeaway}[1][]%
{%
    \refstepcounter{takeaway}
    \tcolorbox[
        enhanced,
        colback=white,
        colframe=white,
        leftrule=0.4mm,
        rightrule=0.4mm,
        toprule=0.4mm,
        bottomrule=0.4mm,
        arc=0mm,
        left=0pt,
        right=0pt,
        top=2pt,
        bottom=2pt,
        breakable,
        borderline north={0.4mm}{0pt}{Takeawaycolor},
        borderline south={0.4mm}{0pt}{Takeawaycolor}
    ]
    \textbf{\textcolor{Takeawaycolor}{\textit{Takeaway~\thetakeaway}}}
    \ifx\relax#1\relax\else~(\textit{#1}).\fi%
}
{%
    \endtcolorbox
}

\newcounter{interpretation}

{%
    \refstepcounter{interpretation}
    \tcolorbox[
        enhanced,
        colback=white,
        colframe=white,
        leftrule=0.4mm,
        rightrule=0.4mm,
        toprule=0.4mm,
        bottomrule=0.4mm,
        arc=0mm,
        left=0pt,
        right=0pt,
        top=2pt,
        bottom=2pt,
        breakable,
        borderline north={0.4mm}{0pt}{Interpretationcolor},
        borderline south={0.4mm}{0pt}{Interpretationcolor}
    ]
    \textbf{\textcolor{Interpretationcolor}{\textit{Interpretation}}}
    \ifx\relax#1\relax\else~(\textit{#1}).\fi%
}
{%
    \endtcolorbox
}

\newcounter{observation}
\renewcommand{\theobservation}{\arabic{observation}}

{%
    \refstepcounter{observation}
    \tcolorbox[
        enhanced,
        colback=white,
        colframe=white,
        leftrule=0.4mm,
        rightrule=0.4mm,
        toprule=0.4mm,
        bottomrule=0.4mm,
        arc=0mm,
        left=0pt,
        right=0pt,
        top=2pt,
        bottom=2pt,
        breakable,
        borderline north={0.4mm}{0pt}{mydarkgreen!80!black},
        borderline south={0.4mm}{0pt}{mydarkgreen!80!black}
    ]
    \textbf{\textcolor{mydarkgreen!80!black}{\textit{Observation~\theobservation}}}
    \ifx\relax#1\relax\else~(\textit{#1}).\fi%
}
{%
    \endtcolorbox
}

\newcounter{promptbox}

\usepackage{pifont}

\usepackage{color, colortbl}
\definecolor{Gray}{gray}{0.93}
\definecolor{Orange}{rgb}{1,0.5,0}
\definecolor{DGray}{gray}{0.83}
\definecolor{LightCyan}{rgb}{0.88,1,1}
\definecolor{pastelgreen}{HTML}{6fc276}
\definecolor{pastelred}{HTML}{ff746c}
\definecolor{pastelpurple}{HTML}{c9a0dc}
\definecolor{pastelorange}{HTML}{ff964f}

\usepackage[T1]{fontenc}

\usepackage{microtype}
\usepackage{graphicx}
\usepackage{subfigure}
\usepackage{booktabs}

\usepackage{wrapfig}
\usepackage{pifont}

\usepackage{blindtext}
\usepackage{lipsum}

\usepackage{multirow}
\usepackage{graphicx}
\usepackage{listings}

\usepackage{bbm}

\usepackage [english]{babel}
\usepackage [autostyle, english = american]{csquotes}
\usepackage{amsmath}
\usepackage{amssymb}
\usepackage{mathtools}
\usepackage{amsthm}

\definecolor{mycitecolor}{HTML}{3498DC}
\definecolor{mylinkcolor}{HTML}{E74D3B}
\definecolor{myurlcolor}{HTML}{980000}
\definecolor{mydarkgreen}{HTML}{6a994e}
\definecolor{myorange}{HTML}{E7730D}
\definecolor{myblue}{HTML}{4594c1}
\usepackage[colorlinks=true,linkcolor=mylinkcolor,citecolor=mycitecolor,urlcolor=myurlcolor]{hyperref}

\usepackage[capitalize,noabbrev]{cleveref}
\usepackage{fontawesome}

\theoremstyle{plain}
\newtheorem{theorem}{Theorem}[section]

\newtheorem{lemma}[theorem]{Lemma}
\newtheorem{corollary}[theorem]{Corollary}
\theoremstyle{definition}
\newtheorem{definition}[theorem]{Definition}
\newtheorem{assumption}[theorem]{Assumption}
\theoremstyle{remark}

\usepackage[textsize=tiny]{todonotes}
\usepackage{thmtools, thm-restate}

\usepackage{pifont}
\usepackage{url}
\usepackage[most]{tcolorbox}
\usepackage{lipsum}
\usepackage{wrapfig}
\usepackage{booktabs}
\usepackage{multirow,mathtools } 

\usepackage{adjustbox}
\MakeOuterQuote{"}

\usepackage{amsmath,amsfonts,bm}

\def\eqref#1{(\ref{#1})}

\def\1{\bm{1}}

\DeclareMathAlphabet{\mathsfit}{\encodingdefault}{\sfdefault}{m}{sl}
\SetMathAlphabet{\mathsfit}{bold}{\encodingdefault}{\sfdefault}{bx}{n}

\usepackage{bm}
\usepackage{dsfont}
\usepackage{mathtools}
\usepackage{nccmath}
\usepackage{amssymb}
\usepackage{multirow}
\usepackage{booktabs}
\usepackage{varwidth}
\usepackage{pifont}
\usepackage{makecell}
\usepackage{wrapfig}
\usepackage{varwidth}
\usepackage{makecell}
\usepackage{enumitem}
\usepackage{adjustbox}
\usepackage{graphicx,calc}
\usepackage{amsfonts,amsthm,bm}
\usepackage[flushleft]{threeparttable}
\usepackage{pifont}
\usepackage{tabularx}
\usepackage{arydshln}
\usepackage{float}
\usepackage{longtable}
\usepackage{ragged2e}
\usepackage{pgfmath}
\usepackage{siunitx}
\usepackage{listings}
\usepackage{rotating}
\usepackage[table]{xcolor}   
\usepackage{collcell}        
\usepackage{etoc}

\title{Selecting Long-Horizon Trajectories for Reliable and Efficient Terminal-Agent Training}

\author{Cuong Dang \& Hoang Anh Just \& Ruoxi Jia \\
Department of Electrical and Computer Engineering \\
Virginia Tech \\
Blacksburg, VA, USA \\
\texttt{\{cuongdc,just,ruoxijia\}@vt.edu}
}

\iclrfinalcopy 
\begin{document}
\etocdepthtag.toc{mtbody}

\maketitle

\begin{abstract}
Terminal agents are commonly trained by imitating long teacher trajectories, yet how much of each trajectory to supervise remains unexplored. We study the \emph{supervision horizon}, the number of trajectory tokens retained for training, and show that it is a key design axis for reliability and cost. Reliability improves with longer horizons but saturates: on Terminal-Bench, a 12K-token horizon solves more tasks than 16K ($29\pm0.7$ vs.\ $26\pm0.8$) while requiring 30\% less training time. The horizon also shapes agent behavior: short horizons cause premature termination, intermediate horizons yield productive error recovery, and long horizons induce over-persistence. We analyze this saturation through a bias--complexity bound, in which longer supervision reduces temporal supervision bias but increases finite-sample estimation error from more heterogeneous late-stage histories. Guided by this analysis, we propose \emph{selective long-horizon refinement}, which first trains on short prefixes and then refines only on continuations that are most likely under the warm-start model. It consistently outperforms full long-horizon training. At 16K, it raises successful attempts from $110\pm2.7$ to $126\pm2.1$ and tasks solved in at least six of eight attempts from $9\pm0.7$ to $14\pm0.6$; with half of the long-horizon data, it still reaches $122\pm2.4$ while cutting training time by 23\%. The gains transfer across benchmarks, from $64\pm2.6$ to $73\pm2.1$ on Terminal-Bench v2.0 and from $137\pm2.7$ to $155\pm2.2$ on OpenThoughts-TBLite. For long-horizon supervision, selecting the right trajectories matters more than training on all of them.

\end{abstract}
\section{Introduction}

\begin{figure}[ht]
    \centering
    \includegraphics[width=.8\linewidth]{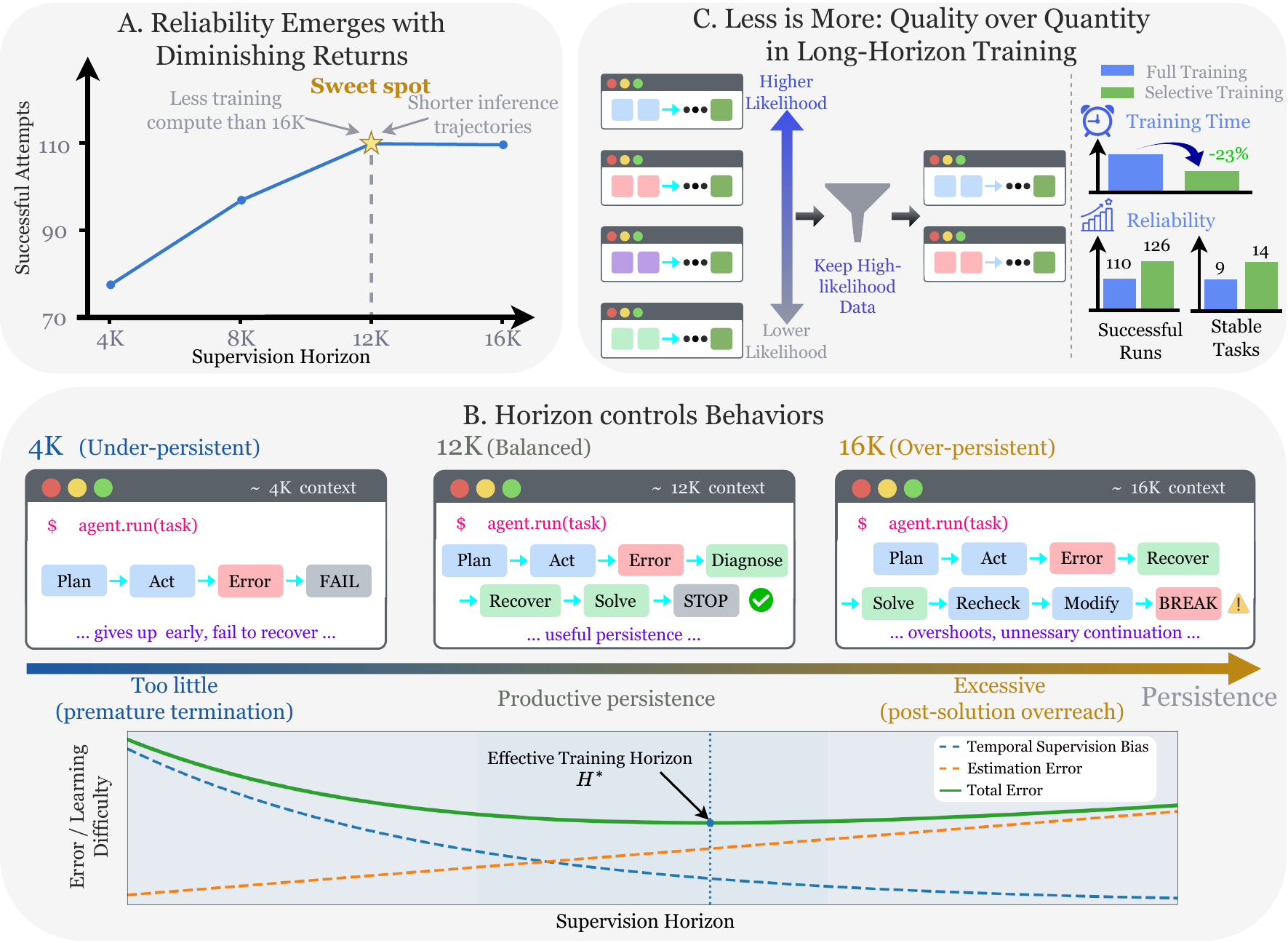}
    \caption{\textbf{Overview of our three main findings on effective training horizons for terminal agents.}\\ \textbf{(A) \textit{Reliability Improves with Diminishing Returns}:} increasing the training horizon substantially improves reliability, but the marginal gains diminish beyond an intermediate sweet spot, where 12K provides a favorable balance between successful attempts, training compute, and inference length.\\ \textbf{(B) \textit{Horizon Controls Behaviors}:} increasing the training horizon shifts agents from under-persistence at 4K, to productive persistence at 12K, and over-persistence at 16K; theoretically, this reflects a trade-off between decreasing temporal supervision bias and increasing estimation difficulty.\\ \textbf{(C) \textit{Less is More: Quality over Quantity in Long-Horizon Training}:} simply training on more long-horizon trajectories is not necessarily beneficial; training on a smaller set of appropriately selected long-horizon trajectories can achieve higher reliability than using the full long-horizon data, showing that data quality matters more than data quantity for exploiting long-horizon supervision.}
    \label{fig:teaser}
\end{figure}

Terminal-based coding agents such as OpenAI Codex~\citep{openai2025codex} and Anthropic Claude Code~\citep{anthropic2026claudeopus48} are becoming a common interface for software development, operating directly over terminals and repositories to execute commands, modify code, and iteratively solve tasks. Recent works have therefore begun to study how to train smaller, capable terminal agents efficiently~\citep{pi2026on,cheng2026terminal,gandhi2026endless,wu2026largescale}.

\definecolor{headerblue}{RGB}{235,242,250}
\definecolor{bestgreen}{RGB}{0,120,90}
\definecolor{rowgray}{RGB}{248,248,248}
\definecolor{sectionteal}{RGB}{20,110,110}
\definecolor{rulegray}{RGB}{100,100,100}
\definecolor{sectionslate}{RGB}{45,80,120}
\definecolor{headerdeep}{RGB}{218,231,246}   
\definecolor{gaingreen}{RGB}{225,243,235}    
\newcommand{\best}[1]{\textbf{\textcolor{bestgreen}{#1}}}
\newcommand{\gain}{\cellcolor{gaingreen}}
\newcommand{\topl}{\textbf{\textcolor{sectionteal}{Top}}}
\newcommand{\botl}{\textcolor{rulegray}{Bot.}}
\newcommand{\saving}[1]{\textcolor{bestgreen}{#1}}
\newcommand{\hz}[1]{\textbf{\textcolor{sectionslate}{#1}}}

\begin{wraptable}{r}{0.45\columnwidth}
\centering
\caption{\textbf{Training time and accuracy} across supervision horizons.}
\label{tab:training_horizon}
\footnotesize
\renewcommand{\arraystretch}{1.1}
\setlength{\tabcolsep}{5pt}
\begin{tabular}{@{}lcccc@{}}
\toprule
\rowcolor{headerblue}
\textbf{Cutoff} & \textbf{4K} & \textbf{8K} & \textbf{12K} & \textbf{16K} \\
\midrule
\rowcolor{rowgray}
Time (h) $\downarrow$ & \textcolor{bestgreen}{\textbf{10.17}} & 21.65 & 35.44 & 50.81 \\
Accuracy (\%) $\uparrow$ & 24 & 25 & \textcolor{bestgreen}{\textbf{29}} & 26 \\
\bottomrule
\end{tabular}
\end{wraptable}

A common way to train terminal agents is imitation learning from teacher-generated trajectories, where the student is supervised to reproduce the teacher’s actions conditioned on the preceding interaction history. These long-horizon terminal trajectories involve extended sequences of interdependent decisions in which later actions depend on earlier observations, modifications, errors, and intermediate outcomes. As trajectories unfold, the agent may need to recover from failed commands, revise hypotheses, verify intermediate results, or switch strategies based on accumulated interaction history. Training with a short horizon can therefore omit these later-stage behaviors and provide insufficient supervision for recovery, correction, verification, and strategy switching. In contrast, longer horizons expose the student to richer temporal dependencies and may improve persistence and long-range reasoning, but they also incur substantially higher training cost and introduce increasingly heterogeneous histories and more complex decision situations. As shown in \Tableref{tab:training_horizon}, increasing the cutoff from 8K to 16K more than \textit{doubles training time}, from 21.65 to 50.81 hours, while Terminal-Bench accuracy improves only marginally from 25 to 26. Moreover, the 12K model achieves the highest accuracy of 29 with 35.44 hours of training, suggesting that simply extending the supervision horizon does not yield proportional performance gains and may even exhibit diminishing returns. This motivates \textbf{two important design questions} that remains largely unresolved:

\begin{quote}
\centering
\emph{\textbf{Q1:} Is supervising the longer teacher trajectory always better?}
\end{quote}

\begin{quote}
\centering
\emph{\textbf{Q2:} Can selectively chosen long-horizon trajectories provide better and more efficient supervision than using all long-horizon data?}
\end{quote}

To answer these questions, we study the \emph{effective training horizon} for terminal agents: a horizon that is sufficiently long to capture useful temporal dependencies, yet not so long that additional supervision becomes inefficient or counterproductive. As summarized in \Figureref{fig:teaser}, \textbf{our study reveals three main findings.} \textbf{\ding{182}} \textit{First}, increasing the training horizon substantially improves reliability, but the gains diminish beyond an intermediate sweet spot, indicating that full-horizon supervision is not necessarily optimal. \textbf{\ding{183}} \textit{Second}, the training horizon systematically shapes the agent’s behavioral regime: short horizons induce under-persistence, intermediate horizons encourage productive persistence, while excessively long horizons lead to over-persistence, causing agents to continue acting unnecessarily and resulting in longer test-time inference trajectories. We further explain this intermediate optimum theoretically by decomposing the total error into two competing quantities: \emph{temporal supervision bias}, which decreases as longer horizons expose the model to richer later-stage behaviors, and \emph{finite-sample estimation error}, which increases as the learner must model a more complex set of long-horizon histories. The effective training horizon therefore emerges at an intermediate point that balances the decreasing temporal supervision bias against the increasing estimation error. These findings lead to a declarative answer to \textbf{Q1}: \textbf{No}, supervising a longer teacher trajectory is not always better; the effective training horizon lies in the middle. \textbf{\ding{184}} \textit{Third}, we show that more long-horizon data is not always better. A two-stage strategy that first trains on the full short-horizon data and then refines the model using only half of the long-horizon trajectories achieves higher reliability than full long-horizon training while reducing both training and inference costs. This leads to a declarative answer to \textbf{Q2}: \textbf{Yes}, selectively chosen long-horizon trajectories can provide more effective and efficient supervision than using the full long-horizon dataset.

\textbf{Our contributions are fourfold} and organized as follows. \textbf{\ding{182}} \textit{First}, we characterize the relationship between training horizon and agent reliability, showing that reliability improves with longer supervision but exhibits diminishing returns (\Sectionref{sec:reliability}). \textbf{\ding{183}} \textit{Second}, we explain this scaling behavior empirically through the notion of persistence, revealing a transition from under-persistence to productive persistence and eventually over-persistence (\Sectionref{sec:horizon_behavior}). \textbf{\ding{184}} \textit{Third}, we theoretically explain why longer supervision eventually exhibits diminishing returns, characterizing the effective training horizon through a trade-off between temporal supervision bias and finite-sample estimation error (\Sectionref{sec:theory_analyze}). \textbf{\ding{185}} \textit{Fourth}, we propose a simple, efficient two-stage training strategy that combines full short-horizon training with selective long-horizon refinement to further improve reliability (\Sectionref{sec:method}).

\section{Related Works}
\label{sec:related_works}

\paragraph{Terminal-Agent Training and Data Engineering.}
Recent work improves terminal agents through data scaling, environment construction, and richer supervision. \citet{gandhi2026endless}, \citet{cheng2026terminal}, and~\citet{shrivastava2026echo} improve environments and supervision signals, while \citet{pi2026on} study long-context training and observe that extending context does not always improve performance. However, \emph{their analysis measures aggregate accuracy and leaves open how longer supervision changes agent reliability}. We show that aggregate accuracy hides a qualitative shift: beyond an intermediate horizon, longer supervision unlocks few new tasks but makes already-solvable tasks more repeatable, redistributing success from fragile to consistent solving. Building on this observation, \textit{we characterize how the supervision horizon shapes agent behavior, from under-persistence to productive recovery and eventually over-persistence, explain the saturation theoretically through a bias--complexity trade-off, and use this analysis to motivate selective refinement of newly exposed long-horizon trajectory segments.}

\paragraph{Likelihood-based data selection for SFT.}
Data quality often matters more than quantity: LIMA~\citep{zhou2023lima} shows that a small, well-chosen dataset can rival much larger ones. One line of work selects data by how well it fits the model being trained. \citet{ren-etal-2024-learn} show that LLM-generated responses outperform human-written ones largely because their lower perplexity makes them easier to learn, and \citet{zhang2025the} select, among candidate responses, those with the highest likelihood under the target model. Other criteria instead favor harder examples, e.g., instruction-following difficulty~\citep{li-etal-2024-quantity}, and perplexity-based pruning in pretraining finds that the best subset depends on the reference model and the retained fraction~\citep{marion2023less,ankner2024perplexed}. Reconciling these views, \citet{cuong2026capacitydependent} show that the preference for high-likelihood data is not universal: in mathematical reasoning, it yields faster early gains, especially for small students, whereas low-likelihood data becomes more beneficial for larger students under longer training. These methods score \emph{whole} examples, typically in single-turn settings. \textit{In contrast, we score only the newly exposed continuation $(H_1, H_2]$ of each long-horizon trajectory, using a warm-start model trained on shorter prefixes, and treat likelihood as compatibility with the learned short-horizon policy rather than as intrinsic quality.}
\section{Background}
\label{sec:background}
\subsection{Terminal Agents as Sequential Decision Making}

In this work, we formulate the interaction following the \emph{Terminus-2} scaffold~\citep{merrill2026terminalbench}, which provides the model with a single headless terminal interface and requires it to complete tasks through Bash commands. This lightweight setup minimizes additional scaffold-specific mechanisms, allowing the learned sequential behavior to be more directly attributed to the underlying language model. At each interaction step $t$, the agent observes the current terminal output $o_t$ and selects an action $a_t$, such as executing a command, inspecting or modifying a file, installing a dependency, or running a test. The action is sampled from a policy $\pi_\theta$ conditioned on the complete interaction history $h_t=(o_1,a_1,\ldots,o_{t-1},a_{t-1},o_t)$, i.e., $a_t\sim\pi_\theta(\cdot\mid h_t)$. A complete interaction forms a trajectory $\tau=(o_1,a_1,o_2,a_2,\ldots,o_T,a_T)$, where $T$ denotes the number of environment-level actions.

\subsection{Terminal Agent Training as Imitation Learning}
\label{sec:background-imitation-learning}
A common paradigm for training terminal agents is to learn from trajectories generated by a stronger teacher model, which can be viewed as \emph{offline imitation learning}. Let $\pi_{\mathrm{T}}$ denote the teacher policy and $\pi_\theta$ the student policy. At each step $t$, the teacher selects an action according to $a_t \sim \pi_{\mathrm{T}}(\cdot \mid h_t)$, where $h_t=(o_1,a_1,\ldots,o_t)$ denotes the interaction history. Rolling out the teacher policy produces a trajectory $\tau=(o_1,a_1,o_2,a_2,\ldots,o_T,a_T)$. A dataset of $N$ teacher trajectories is denoted by $\mathcal{D}=\{\tau^{(i)}\}_{i=1}^{N}$, where $\tau^{(i)}=(o_1^{(i)},a_1^{(i)},\ldots,o_{T_i}^{(i)},a_{T_i}^{(i)})$.

Given these offline demonstrations, behavior cloning trains the student to imitate the teacher action at each visited history. The student assigns probability $\pi_\theta(a_t^{(i)}\mid h_t^{(i)})$ to the teacher action observed at step $t$, with the standard log-loss objective
\begin{equation}
    \mathcal{L}(\theta)=-\frac{1}{N}\sum_{i=1}^{N}\sum_{t=1}^{T_i}\log\pi_\theta\left(a_t^{(i)} \mid h_t^{(i)}\right).
\end{equation}

\subsection{From Decision Horizon to Token-Level Supervision}
\label{sec:sft}
In practice, LLM-based agents are trained with supervised fine-tuning (SFT) over tokenized interaction sequences. In SFT, each action $a_t=(y_{t,1},\ldots,y_{t,\ell_t})$ is generated autoregressively at the token level. Let $p_\theta(y_{t,j}\mid h_t,y_{t,<j})$ denote the next-token distribution. Then,
\begin{equation}
    \pi_\theta(a_t\mid h_t)
    =
    \prod_{j=1}^{\ell_t}
    p_\theta(y_{t,j}\mid h_t,y_{t,<j}),
    \qquad
    \log \pi_\theta(a_t\mid h_t)
    =
    \sum_{j=1}^{\ell_t}
    \log p_\theta(y_{t,j}\mid h_t,y_{t,<j}).
\end{equation}

A common cutoff at a particular decision step does not correspond to a common sequence length across trajectories. Instead, the supervision horizon is implemented as a token budget $K$, retaining only the prefix of the serialized trajectory within the first $K$ tokens. Letting $n_t(\tau)$ denote the cumulative number of tokens through action $a_t$, a token budget $K$ induces the decision horizon $H_K(\tau)=\max\{t:n_t(\tau)\leq K\}$, which is non-decreasing in $K$. Throughout the empirical analysis, we therefore use the \emph{token supervision horizon} $K$ as the controllable training variable, while the theoretical analysis is expressed in terms of the induced \emph{decision horizon} $H$.
\section{Emergence and Saturation of Reliability}
\label{sec:reliability}

We first investigate how the amount of trajectory supervision affects the reliability of terminal agents. We vary only the token-level supervision horizon while keeping the base model, underlying training trajectories, and evaluation protocol fixed. The resulting trend is clear: extending the supervision horizon initially improves both the breadth and repeatability of task solving, but the gains diminish beyond an intermediate regime. In particular, substantially longer supervision incurs considerably higher training cost without a corresponding improvement in aggregate reliability.

\subsection{Experimental Setting}

\textbf{Model and Training Data.}
Following prior work on terminal-agent training~\citep{pi2026on,gandhi2026endless,cheng2026terminal,shrivastava2026echo}, we use \texttt{Qwen/Qwen3-8B} as the base model throughout our experiments. We train on trajectories from \texttt{nvidia/Nemotron-Terminal-Corpus}, which contains terminal-interaction trajectories generated by DeepSeek-V3.2~\citep{liu2025deepseek}. Following \citet{pi2026on}, we restrict training to completed trajectories, allowing us to isolate the effect of supervision horizon from confounding effects introduced by incomplete or imperfect trajectories. All training variants are constructed from the same underlying trajectory set, with supervision truncated at $4096$, $8192$, $12288$, or $16384$ response tokens. We refer to these settings as $4$K, $8$K, $12$K, and $16$K, respectively.

\textbf{Evaluation Protocol.} We evaluate the resulting models on Terminal-Bench v1.0~\citep{merrill2026terminalbench}, using the Terminus-2 agent scaffold throughout. Each task is evaluated with eight independent attempts. Repeated evaluation allows us to distinguish between \emph{coverage} --- whether a model can solve a task at all --- and \emph{repeatability} --- whether it can solve the same task consistently across independent runs.

\textbf{Reliability Metrics.}
Let $s_i=\sum_{r=1}^{8}\mathbf{1}\{\text{attempt }r\text{ succeeds on task }i\}$ be the number of successes for task $i$. We define
\begin{equation}
    \mathrm{Rel}_{\geq k}
    =
    \sum_{i=1}^{M}\mathbf{1}\{s_i\geq k\},
    \qquad
    \bar{s}_{+}
    =
    \frac{\sum_{i:s_i\geq1}s_i}{\mathrm{Rel}_{\geq1}},
\end{equation}
where $\mathrm{Rel}_{\geq k}$ counts tasks solved in at least $k$ of eight attempts, and $\bar{s}_{+}$ is the mean success count among solved tasks. We also report fragile tasks, defined by $1\leq s_i\leq2$.


\subsection{Experimental Results}

\begin{wraptable}{r}{0.5\textwidth}
\vspace{-1.6em}
\centering
\caption{\textbf{Reliability across supervision horizons.}
  Values are mean\,$\pm$\,variation over four evaluations.
  }
\label{tab:horizon_reliability}

\scriptsize
\renewcommand{\arraystretch}{0.9}
\setlength{\tabcolsep}{1.5pt}

\begin{tabularx}{\linewidth}{@{}l*{4}{>{\centering\arraybackslash}X}@{}}
\toprule
\rowcolor{headerblue}
\textbf{Metric} & \textbf{4K} & \textbf{8K} & \textbf{12K} & \textbf{16K} \\
\midrule

\multicolumn{5}{c}{\textcolor{sectionteal}{\textit{Aggregate}}} \\
Succ.\ $\uparrow$              & $78{\pm}2.8$   & $97{\pm}2.6$   & \textcolor{bestgreen}{\boldmath$111{\pm}2.3$} & $110{\pm}2.7$ \\
$\mathrm{Rel}_{\geq1}\uparrow$ & $24{\pm}0.8$   & $25{\pm}0.8$   & \textcolor{bestgreen}{\boldmath$29{\pm}0.7$}  & $26{\pm}0.8$  \\
$\bar{s}_{+}\uparrow$          & $3.25{\pm}0.1$ & $3.88{\pm}0.1$ & $3.83{\pm}0.1$ & \textcolor{bestgreen}{\boldmath$4.23{\pm}0.1$} \\
\midrule

\multicolumn{5}{c}{\textcolor{sectionteal}{\textit{Repeatability}}} \\
$\mathrm{Rel}_{\geq4}\uparrow$ & $9{\pm}0.7$    & $10{\pm}0.7$   & $13{\pm}0.8$   & \textcolor{bestgreen}{\boldmath$14{\pm}0.8$} \\
$\mathrm{Rel}_{\geq6}\uparrow$ & $4{\pm}0.5$    & $8{\pm}0.6$    & \textcolor{bestgreen}{\boldmath$9{\pm}0.6$} & \textcolor{bestgreen}{\boldmath$9{\pm}0.7$} \\
$\mathrm{Rel}_{\geq8}\uparrow$ & $2{\pm}0.4$    & $5{\pm}0.5$    & \textcolor{bestgreen}{\boldmath$6{\pm}0.5$} & $4{\pm}0.5$ \\
Fragile $\downarrow$           & $15{\pm}0.8$   & $11{\pm}0.7$   & $12{\pm}0.8$   & \textcolor{bestgreen}{\boldmath$9{\pm}0.7$} \\
\bottomrule
\end{tabularx}
\vspace{-1.8em}
\end{wraptable}

\textbf{Reliability emerges with longer supervision.}
As shown in \Tableref{tab:horizon_reliability}, increasing the supervision horizon from $4$K to $12$K substantially improves reliability. Successful attempts increase from $78{\pm}2.8$ to $111{\pm}2.3$, while task coverage rises from $\mathrm{Rel}_{\geq1}=24{\pm}0.8$ to $29{\pm}0.7$. The improvement also extends to repeated success: $\mathrm{Rel}_{\geq6}$ increases from $4{\pm}0.5$ to $9{\pm}0.6$, and $\mathrm{Rel}_{\geq8}$ from $2{\pm}0.4$ to $6{\pm}0.5$. In every case, the gain far exceeds the variation across evaluations. Thus, longer supervision improves both the breadth and consistency of task solving.

\textbf{Reliability saturates beyond an intermediate horizon.}
As shown in \Tableref{tab:horizon_reliability}, extending supervision from $12$K to $16$K yields little additional aggregate benefit. Successful attempts remain nearly unchanged ($111{\pm}2.3$ vs.\ $110{\pm}2.7$, within variation), while task coverage decreases from $29{\pm}0.7$ to $26{\pm}0.8$ and $\mathrm{Rel}_{\geq8}$ from $6{\pm}0.5$ to $4{\pm}0.5$. However, the reliability profile shifts: the number of fragile tasks decreases from $12{\pm}0.8$ to $9{\pm}0.7$, while $\mathrm{Rel}_{\geq4}$ increases slightly from $13{\pm}0.8$ to $14{\pm}0.8$. Together with the higher mean success count ($\bar{s}_{+}=3.83{\pm}0.1$ vs.\ $4.23{\pm}0.1$), this suggests that after the reliability peak, additional supervision primarily \emph{redistributes} success toward more consistent performance on already-solvable tasks rather than expanding overall coverage.

We defer task-level reliability and further analyses in \Appendixref{sec:task_level_reliability}.

\begin{takeaway}[Emergence-and-Saturation of Reliability]
Reliability improves markedly up to an intermediate horizon, after which additional supervision mainly redistributes success from fragile to more repeatable task solving rather than increasing aggregate reliability.
\end{takeaway}
\section{Horizon controls Behaviors}
\label{sec:horizon_vs_behaviors}

\subsection{Empirical Analysis}
\label{sec:horizon_behavior}

\definecolor{bestgreen}{HTML}{2E7D32}

\begin{wraptable}{r}{0.45\linewidth}
    \vspace{-2.5em}
    \centering
    \caption{\textbf{Proactive vs.\ reactive} successful trajectories across four context lengths. Proactive share peaks at 12K and declines at 16K.}
    \label{tab:proactive_reactive}
    \resizebox{\linewidth}{!}{
    \begin{tabular}{l*{4}{c}}
        \arrayrulecolor{rulegray}\toprule
        \rowcolor{headerblue}
        \textbf{Context length} & \textbf{4K} & \textbf{8K} & \textbf{12K} & \textbf{16K} \\
        \midrule
        Proactive & 17 & 25 & 29 & 27 \\
        \rowcolor{rowgray}
        Reactive  & 61 & 72 & 82 & 83 \\
        \midrule
        Proactive share
        & \cellcolor{sectionslate!12}\textcolor{sectionslate}{\textbf{21.79\%}}
        & \cellcolor{sectionslate!32}\textcolor{sectionslate}{\textbf{25.77\%}}
        & \cellcolor{bestgreen!25}\textcolor{bestgreen}{\textbf{26.13\%}}
        & \cellcolor{sectionslate!24}\textcolor{sectionslate}{\textbf{24.55\%}} \\
        \bottomrule
    \end{tabular}
    }
    \vspace{-1em}
\end{wraptable}

\textbf{Longer horizons strengthen reactive problem solving.}
We classify a successful trajectory as proactive if the target failure never manifests, and as reactive if the failure manifests but the agent subsequently diagnoses and recovers from it to solve the task; details in \Sectionref{sec:proactive_reactive}. As shown in \Tableref{tab:proactive_reactive}, increasing the supervision horizon consistently increases the number of successful reactive trajectories, from 61 at 4K to 83 at 16K. This suggests that longer-horizon supervision supports successful diagnosis and recovery after errors occur. Proactive successes also increase from 17 at 4K to 29 at 12K before slightly decreasing to 27 at 16K, with the proactive share peaking at 26.13\% at 12K. Overall, longer horizons are associated with more successful recovery trajectories while preserving strong proactive solving.

\textbf{Longer horizons can induce over-persistence.} Together, \Tableref{tab:proactive_reactive} and \Figureref{fig:episodes_per_run} suggest that 12K provides a favorable balance between proactive and reactive behavior. At 12K, proactive successes peak at 29 (26.13\% of successful trajectories), while reactive successes reach 82. It also produces the shortest successful trajectories, averaging 10.8 episodes overall and 12.6 episodes among reactive runs. Scaling to 16K yields only a marginal increase in reactive successes (82$\rightarrow$83), while proactive successes decrease (29$\rightarrow$27) and successful trajectories become longer on average, both overall (10.8$\rightarrow$12.7 episodes) and among reactive runs (12.6$\rightarrow$15.1 episodes). These results suggest that the additional persistence induced by longer supervision becomes less efficient beyond 12K and is consistent with over-persistence at 16K. We leave detailed task-specific analyses in \Appendixref{sec:persistence}.

\begin{figure}[t]
    \centering
    \vspace{-1.5em}
    \caption{\textbf{Episodes per run for successful trials, all and reactive only across training horizons.}}
    \includegraphics[width=.9\linewidth]{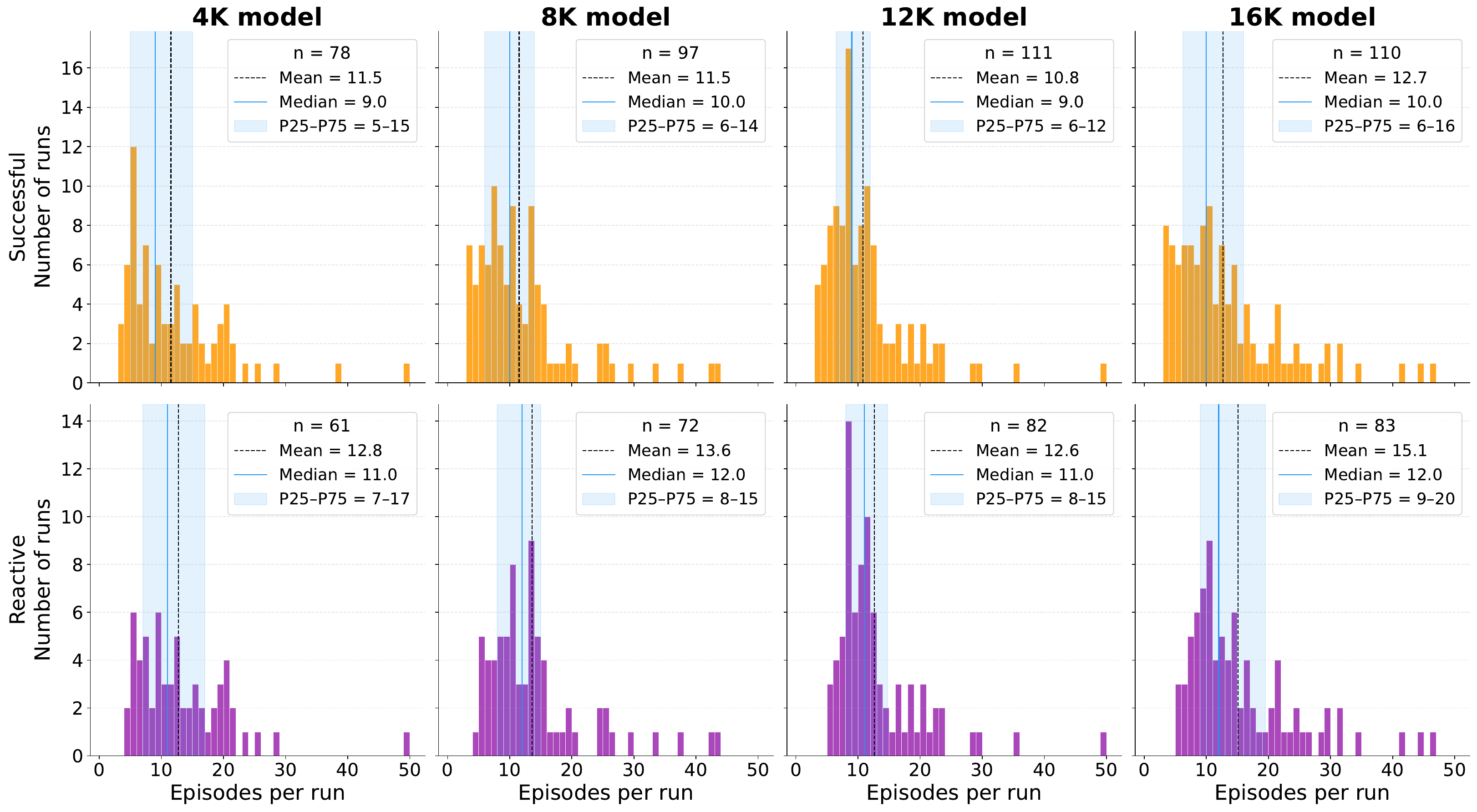}
    \label{fig:episodes_per_run}
    \vspace{-1em}
\end{figure}

\begin{takeaway}[Persistence Emerges and Eventually Overshoots]
    Increasing the supervision horizon is associated with more successful recovery trajectories, with the most favorable balance around 12K. Further scaling provides little additional reactive success while producing longer trajectories, suggesting over-persistence at 16K.
\end{takeaway}

\subsection{Why Does Reliability Saturate?}
\label{sec:theory_analyze}

Longer supervision exposes later-stage behavior absent under short-horizon training, but these stages involve increasingly diverse histories. We formalize this tension as a trade-off between temporal supervision bias and finite-sample estimation difficulty. Let $L_H(\pi)$ be the expected imitation loss over the first $H$ teacher decisions, and $\widehat{\pi}_H$ the empirical risk minimizer at horizon $H$.

\begin{definition}[Temporal supervision bias]
\label{def:temporal_bias}
Let $\Pi_H^\star$ be the set of prefix-optimal policies and $\pi_T^\star$ a full-horizon optimal policy. The temporal supervision bias at horizon $H$ is
\begin{equation}
\begin{gathered}
    \Pi_H^\star=\arg\min_{\pi\in\Pi}L_H(\pi),
    \qquad
    \pi_T^\star\in\arg\min_{\pi\in\Pi}L_T(\pi),\\
    B(H):=\sup_{\pi\in\Pi_H^\star}L_T(\pi)-L_T(\pi_T^\star),
    \qquad
    B(H)\geq 0,\quad B(T)=0.
\end{gathered}
\end{equation}
\end{definition}

$B(H)$ is the worst-case full-horizon risk left undetermined by supervising only the first $H$ decisions. We assume it is non-increasing in $H$, while the horizon-induced policy complexity $\mathcal{C}(H)$ is non-decreasing (\Lemmaref{lem:complexity}).

\begin{theorem}[Bias--Complexity Bound]
\label{thm:bias_complexity_bound}
Assume $N$ independent teacher trajectories and a bounded loss $0\leq\ell\leq M$. Let $L_H^\star:=\min_{\pi\in\Pi}L_H(\pi)$, and assume the prefix-to-full-horizon transfer coefficient is bounded:
\begin{equation}
\label{eq:transfer_coefficient}
    \Gamma_H
    :=
    \sup_{\pi\in\Pi:\,L_H(\pi)>L_H^\star}
    \frac{\bigl[L_T(\pi)-\sup_{\pi'\in\Pi_H^\star}L_T(\pi')\bigr]_+}{L_H(\pi)-L_H^\star}
    \leq\Gamma<\infty,
    \qquad \forall H\leq T.
\end{equation}
Then, with probability at least $1-\delta$,
\begin{equation}
\label{eq:bias_complexity_bound}
    L_T(\widehat{\pi}_H)-L_T(\pi_T^\star)
    \leq
    B(H)
    +
    \sqrt{2}\,\Gamma M
    \sqrt{\frac{\mathcal{C}(H)+\log(2/\delta)}{N}}.
\end{equation}
\end{theorem}

The proof, assumptions, and scope are in \Appendixref{appdendix:proofs}. The result bounds imitation risk rather than task success, so we use it as a lens on the trade-off and test it empirically.

\begin{wrapfigure}{r}{0.45\linewidth}
    \centering
    \vspace{-2em}
    \caption{
        \textbf{Supervision Bias $\widehat{B}(\cdot)$ and Temporal Difficulty $\ell^{\mathrm{loc}}_{H}(\cdot;4)$ for the 8K, 12K and 16K models.}
    }
    \includegraphics[width=\linewidth]{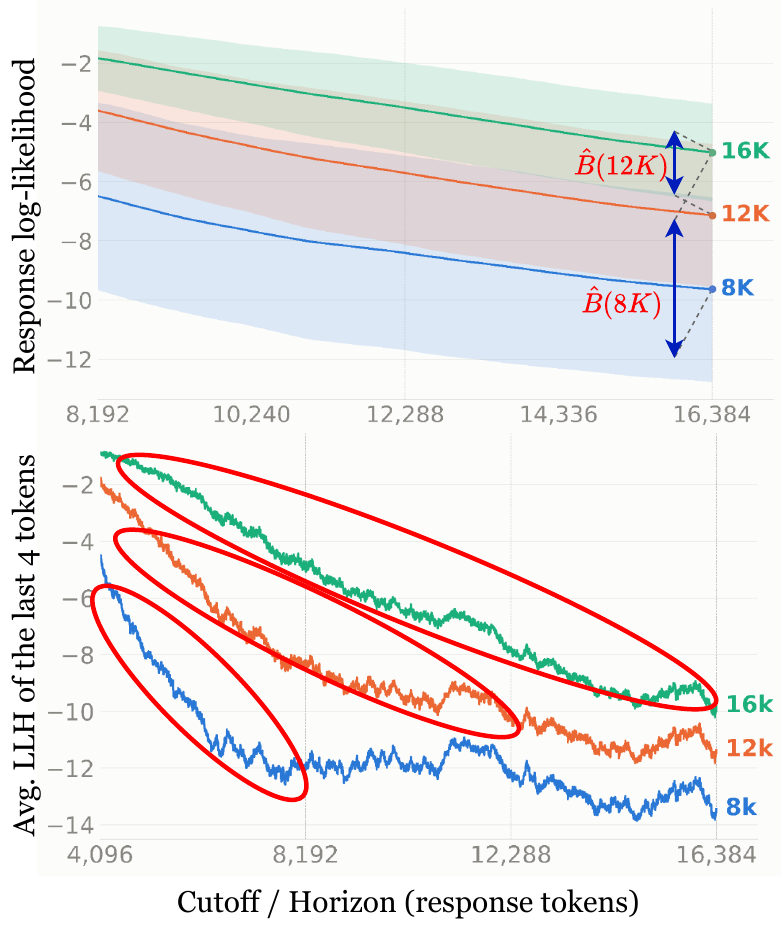}
    \vspace{-7em}
    \label{fig:temporal_likelihood}
\end{wrapfigure}

\textbf{Empirical diagnostics.} The two terms in \Theoremref{thm:bias_complexity_bound} are not directly observable, so we measure one diagnostic for each on training trajectories longer than $16$K tokens. For each horizon-$H$ model $\theta_H$, we compute
\begin{equation}
\label{eq:temporal_metrics}
\begin{aligned}
    \bar{\ell}_{H}(K) &= \frac{1}{K}\sum_{t\leq K}\log p_{\theta_H}(x_t\mid x_{<t}),\\
    \widehat{B}(H) &= \bar{\ell}_{16K}(K_{\max})-\bar{\ell}_{H}(K_{\max}),\\
    \ell^{\mathrm{loc}}_{H}(K;w) &= \frac{1}{w}\sum_{t=K+1}^{K+w}\log p_{\theta_H}(x_t\mid x_{<t}),
\end{aligned}
\end{equation}
where $K_{\max}=16$K and we use a local window of $w=4$ tokens throughout. The bias proxy $\widehat{B}(H)$ targets the first term, measuring how much worse a model supervised up to $H$ fits the full $16$K trajectory than one supervised on all of it, thereby mirroring the truncation-induced excess risk in $B(H)$. The local difficulty $\ell^{\mathrm{loc}}_H(K;4)$ targets the second term: tracking a short window's likelihood as $K$ increases reveals how difficult each position becomes to model, reflecting the growing heterogeneity of late-stage histories that drives estimation burden. These quantities are diagnostics, not direct estimates of the theoretical terms (\Sectionref{sec:empirical_measures}).

As shown in \Figureref{fig:temporal_likelihood}, $\ell^{\mathrm{loc}}_{H}(K;4)$ drops substantially as the trajectory progresses, even within supervised regions (\textcolor{red}{red oval regions}), so later-stage behavior is harder to model. Meanwhile, $\widehat{B}(12K)<\widehat{B}(8K)$: longer supervision narrows the gap to the $16$K reference, consistent with the assumption that temporal supervision bias is non-increasing in $H$. Together with \Theoremref{thm:bias_complexity_bound}, this suggests that longer supervision reduces truncation bias but adds behavior that is harder to estimate from finite data.

\textbf{Interpretation.}
\Theoremref{thm:bias_complexity_bound} and \Figureref{fig:temporal_likelihood} point to the same tension: longer supervision reduces the deficit caused by truncating later-stage behavior, but the newly introduced behavior is harder to estimate from finite data. This is consistent with the saturation observed in \Sectionref{sec:reliability}, and it motivates the selective refinement strategy in \Sectionref{sec:method}. The bound does not explain the behavioral shift toward over-persistence, which we regard as an empirical finding.

\begin{takeaway}[Estimation Burden]
    Longer supervision reduces temporal bias but increases estimation burden by exposing the model to more heterogeneous late-stage behaviors. Empirically, supervision beyond the effective horizon coincides with over-persistence (\Sectionref{sec:horizon_behavior}).
\end{takeaway}
\section{Selective Long-Horizon Refinement}
\label{sec:method}

\textbf{Motivation \& Method Overview.}
Longer supervision exposes useful later-stage behaviors, but training on all long-horizon continuations does not consistently improve reliability. Motivated by the bias--complexity trade-off in \Theoremref{thm:bias_complexity_bound} and increasing temporal difficulty in \Figureref{fig:temporal_likelihood}, we propose a two-stage selective refinement strategy. We first train on all trajectories up to a shorter horizon $H_1$, then extend supervision to $H_2$ using continuations with high likelihood under the warm-start model. We interpret likelihood as \emph{compatibility} with behavior learned at $H_1$, rather than as a universal measure of trajectory quality. This provides a conservative way to introduce long-horizon behavior and improve the reliability--efficiency trade-off without uniformly training on difficult continuations.

\begin{figure}[ht]
    \centering
    \includegraphics[width=\linewidth]{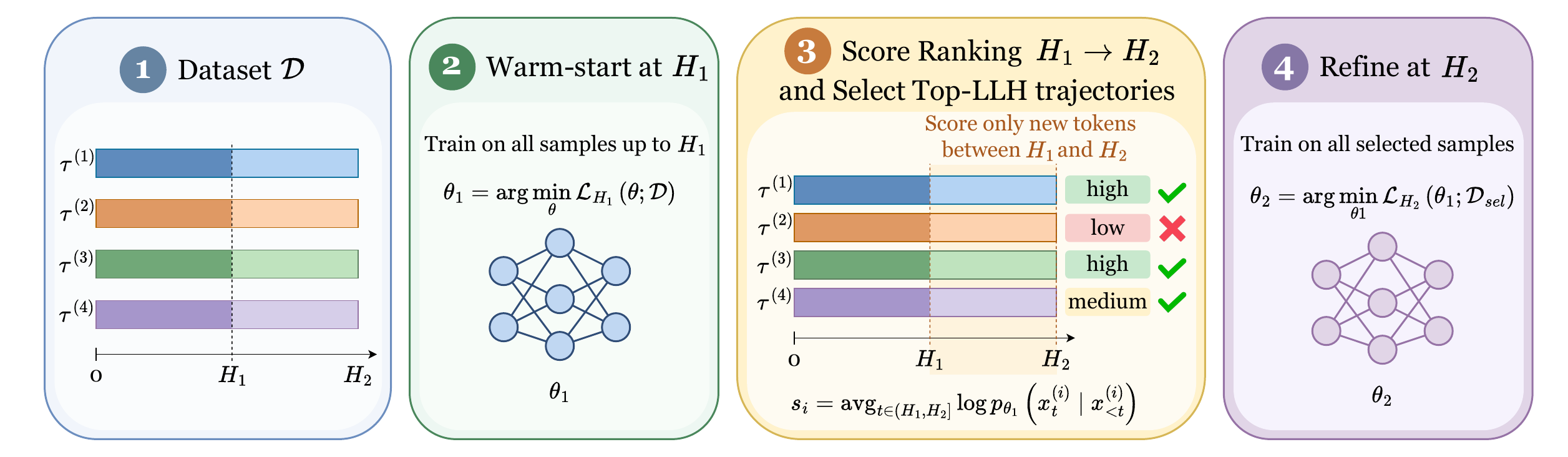}
    \vspace{-1em}
    \caption{\textbf{Overview of the proposed two-stage selective long-horizon training method.}}
    \label{fig:method}
\end{figure}

\textbf{Stage I: Data Preparation.} Following \Sectionref{sec:background}, let
$\mathcal{D}=\{\tau^{(i)}\}_{i=1}^{N}$ denote the offline teacher dataset,
where each trajectory
$\tau^{(i)}=(x_1^{(i)},\ldots,x_{L_i}^{(i)})$
contains $L_i$ supervised response tokens. We consider two supervision
horizons, $H_1<H_2$, where $H_1$ is used for short-horizon warm-start
training and $H_2$ for subsequent long-horizon refinement. Since scoring
requires observing the complete continuation from $H_1$ to $H_2$, we restrict
the selection pool to trajectories of length at least $H_2$,
$\mathcal{D}_{H_2}=\{\tau^{(i)}\in\mathcal{D}:L_i\geq H_2\}$.
For each candidate trajectory, the interval $(H_1,H_2]$ is used for
likelihood-based scoring and selection.

\textbf{Stage II: Short-Horizon Warm Start.} We train on the full dataset with supervision truncated at $H_1$:
\begin{equation}
\label{eq:token-average}
    \mathcal{L}_{H_1}(\theta;\mathcal{D})
    =
    -\frac{
    \sum_{i=1}^{N}\sum_{t=1}^{\min(H_1,L_i)}
    \log p_{\theta}(x_t^{(i)}\mid x_{<t}^{(i)})
    }{
    \sum_{i=1}^{N}\min(H_1,L_i)
    },
    \qquad
    \theta_1
    =
    \arg\min_{\theta}
    \mathcal{L}_{H_1}(\theta;\mathcal{D}).
\end{equation}
The resulting model $\theta_1$ provides both the initialization for
long-horizon training and the reference model for selecting learnable
continuations.

\textbf{Stage III: Continuation Scoring and Selection.} For each $\tau^{(i)}\in\mathcal{D}_{H_2}$, we score the continuation between
$H_1$ and $H_2$ using the warm-start model:
\begin{equation}
    s_i(H_1,H_2)
    =
    \frac{1}{H_2-H_1}
    \sum_{t=H_1+1}^{H_2}
    \log p_{\theta_1}(x_t^{(i)}\mid x_{<t}^{(i)}).
    \label{eq:continuation_score}
\end{equation}
Given a selection ratio $\rho$, we retain the highest-scoring trajectories
to form $\mathcal{D}_{\mathrm{sel}}\subseteq\mathcal{D}_{H_2}$.

\textbf{Stage IV: Selective Long-Horizon Refinement.} Starting from the warm-start checkpoint $\theta_1$, we continue training on
$\mathcal{D}_{\mathrm{sel}}$ with the supervision horizon extended to $H_2$.
Since all selected trajectories satisfy $L_i\geq H_2$, the refinement objective is
\begin{equation}
    \mathcal{L}_{H_2}(\theta;\mathcal{D}_{\mathrm{sel}})
    =
    -\frac{1}{|\mathcal{D}_{\mathrm{sel}}|H_2}
    \sum_{\tau^{(i)}\in\mathcal{D}_{\mathrm{sel}}}
    \sum_{t=1}^{H_2}
    \log p_{\theta}
    \left(
        x_t^{(i)} \mid x_{<t}^{(i)}
    \right),
    \qquad
    \theta_2
    =
    \arg\min_{\theta\,;\,\theta\leftarrow\theta_1}
    \mathcal{L}_{H_2}(\theta;\mathcal{D}_{\mathrm{sel}}).
\end{equation}
\section{Experiments}
\label{sec:experiments}

We compare selective long-horizon refinement with one-stage training at 4K, 8K, 12K, and 16K. For 12K and 16K refinement, we warm-start from the 8K checkpoint, rank candidate continuations by likelihood under the warm-start model, and retain the top or bottom $25\%$, $50\%$, or $75\%$for training. All models are trained for two epochs, with each stage of two-stage training run for two epochs. Following \Sectionref{sec:reliability}, we report successful attempts, $\mathrm{Rel}{\geq 1}$, $\mathrm{Rel}{\geq 4}$, $\mathrm{Rel}{\geq 6}$, $\mathrm{Rel}{\geq 8}$, $\bar{s}_{+}$, and fragile tasks. For each checkpoint, we evaluate eight attempts per task, reporting the mean and run-to-run variation across four evaluations. Our main results use Terminal-Bench v1.0, while results on Terminal-Bench v2.0 and OpenThoughts-TBLite~\citep{merrill2026terminalbench} are deferred to \Appendixref{sec:selective_other_benchmarks}.

\subsection{Main Results}
\label{sec:main_results}

\begin{wraptable}{r}{0.5\linewidth}
\vspace{-2em}
\centering
\caption{\textbf{Training cost of one-stage and selective two-stage training.}
\topl-50\% models are warm-started from 8K; \emph{Total} includes the warm-start stage.
Parentheses give the change vs.\ one-stage training at the same horizon.}
\label{tab:training_cost}

\footnotesize
\renewcommand{\arraystretch}{1.12}

\resizebox{\linewidth}{!}{%
\begin{tabular}{llrrr}
\toprule
\rowcolor{headerdeep}
\textbf{Hor.} &
\textbf{Strategy} &
\textbf{\#Samples} &
\textbf{Time (h)} &
\textbf{Total (h)}$\downarrow$ \\
\midrule

\hz{8K}
& Full
& 59,408
& 21.65
& 21.65 \\

\midrule
\hz{12K}
& Full
& 59,408
& 35.44
& 35.44 \\

\rowcolor{rowgray}
\hz{12K}
& \topl-50\%
& \gain 19,568
& \gain 11.83
& \gain \best{33.48}\,{\scriptsize\saving{($-$5.5\%)}} \\

\midrule
\hz{16K}
& Full
& 59,408
& 50.81
& 50.81 \\

\rowcolor{rowgray}
\hz{16K}
& \topl-50\%
& \gain 19,600
& \gain 17.48
& \gain 39.13\,{\scriptsize\saving{($-$23.0\%)}} \\

\bottomrule
\end{tabular}%
}

\vspace{-2em}
\end{wraptable}

\paragraph{Selective refinement improves long-horizon training.}
With $\rho \geq 50\%$, Top-$\rho$ refinement improves over one-stage training at the same horizon (\Tableref{tab:all_selective_results}). At 12K, Top-50 increases successful attempts from $111{\pm2.3}$ to $121{\pm2.3}$, raises $\mathrm{Rel}_{\geq6}$ from $9{\pm0.6}$ to $11{\pm0.6}$, and reduces fragile tasks from $12{\pm0.8}$ to $9{\pm0.7}$. At 16K, Top-75 performs best, increasing successful attempts from $110{\pm2.7}$ to $126{\pm2.1}$, $\mathrm{Rel}_{\geq6}$ from $9{\pm0.7}$ to $14{\pm0.6}$, and $\mathrm{Rel}_{\geq8}$ from $4{\pm0.5}$ to $7{\pm0.5}$. These gains exceed run-to-run variation, showing that selective long-horizon supervision is more effective than uniform training.

\paragraph{Selective refinement also reduces training cost.}
The refinement stage uses only a third of the samples of one-stage training (19{,}568 vs.\ 59{,}408; \Tableref{tab:training_cost}).
Including the 8K warm start, 8K$\rightarrow$12K \topl-50\% takes 33.48\,h in total, less than one-stage 12K (35.44\,h) and 16K (50.81\,h) training, while solving more attempts (121 vs.\ 111 and 110).
At 16K, selective refinement cuts total training time by 23.0\% (39.13 vs.\ 50.81\,h).

\subsection{Ablation Study}
\label{sec:ablation}

\definecolor{gaingreen}{HTML}{E3F4E6}
\definecolor{bestgreen}{HTML}{1B7F3B}
\definecolor{headerdeep}{HTML}{DCE6F2}
\definecolor{headerblue}{HTML}{EEF3FA}
\definecolor{rowgray}{HTML}{F6F7F9}
\definecolor{rulegray}{HTML}{8A9099}
\definecolor{sectionslate}{HTML}{34495E}
\definecolor{sectionteal}{HTML}{1F7A7A}

\begin{table}[t]
\centering
\vspace{-1em}
\caption{
\textbf{Reliability of one-stage and selective two-stage training on Terminal-Bench v1.0.}
Two-stage models warm-start from 8K and refine at 12K or 16K on the highest- (\topl) or lowest-likelihood (\botl) $\rho\%$ of continuations.
Values are reported with estimated run-to-run variation.
\colorbox{gaingreen}{Shaded}: improves over one-stage training at the same horizon;
\textbf{\textcolor{bestgreen}{green}}: best in row.
}
\label{tab:all_selective_results}
\setlength{\tabcolsep}{3.2pt}
\setlength{\aboverulesep}{0pt}
\setlength{\belowrulesep}{0pt}
\renewcommand{\arraystretch}{1.18}
\arrayrulecolor{rulegray}
\resizebox{\textwidth}{!}{%
\begin{tabular}{@{}l cccc cccccc cccccc@{}}
\toprule
\rowcolor{headerdeep}
& \multicolumn{4}{c}{\textbf{\textcolor{sectionslate}{One-stage}}}
& \multicolumn{6}{c}{\textbf{\textcolor{sectionslate}{12K refinement}} (from 8K)}
& \multicolumn{6}{c}{\textbf{\textcolor{sectionslate}{16K refinement}} (from 8K)} \\
\cmidrule(lr){2-5}\cmidrule(lr){6-11}\cmidrule(lr){12-17}

\rowcolor{headerblue}
& \multicolumn{4}{c}{\textit{training horizon}}
& \multicolumn{2}{c}{$\rho{=}25\%$}
& \multicolumn{2}{c}{$\rho{=}50\%$}
& \multicolumn{2}{c}{$\rho{=}75\%$}
& \multicolumn{2}{c}{$\rho{=}25\%$}
& \multicolumn{2}{c}{$\rho{=}50\%$}
& \multicolumn{2}{c}{$\rho{=}75\%$} \\
\cmidrule(lr){6-7}\cmidrule(lr){8-9}\cmidrule(lr){10-11}
\cmidrule(lr){12-13}\cmidrule(lr){14-15}\cmidrule(lr){16-17}

\rowcolor{headerblue}
\textbf{Metric}
& 4K & 8K & 12K & 16K
& \topl & \botl & \topl & \botl & \topl & \botl
& \topl & \botl & \topl & \botl & \topl & \botl \\

\midrule
\multicolumn{17}{@{}l}{
\textbf{\textcolor{sectionteal}{\textit{Aggregate reliability}}}
} \\

Succ.\ attempts $\uparrow$
& $78{\pm2.8}$ & $97{\pm2.6}$ & $111{\pm2.3}$ & $110{\pm2.7}$
& \gain $125{\pm2.2}$ & $108{\pm2.8}$
& \gain $121{\pm2.3}$ & \gain $112{\pm2.6}$
& \gain $119{\pm2.5}$ & \gain $117{\pm2.7}$
& \gain $114{\pm2.6}$ & $108{\pm2.9}$
& \gain $122{\pm2.4}$ & \gain $120{\pm2.5}$
& \gain \textcolor{bestgreen}{\boldmath$126{\pm2.1}$} & \gain $114{\pm2.7}$ \\

\rowcolor{rowgray}
$\mathrm{Rel}_{\geq1}$ $\uparrow$
& $24{\pm0.8}$ & $25{\pm0.8}$ & $29{\pm0.7}$ & $26{\pm0.8}$
& $26{\pm0.7}$ & $28{\pm0.8}$
& $27{\pm0.7}$ & $28{\pm0.8}$
& $29{\pm0.7}$ & \gain \textcolor{bestgreen}{\boldmath$31{\pm0.8}$}
& \gain $27{\pm0.7}$ & $24{\pm0.8}$
& \gain $28{\pm0.7}$ & $25{\pm0.8}$
& $26{\pm0.7}$ & \gain $30{\pm0.8}$ \\

$\bar{s}_{+}$ $\uparrow$
& $3.25{\pm0.1}$ & $3.88{\pm0.1}$ & $3.83{\pm0.1}$ & $4.23{\pm0.1}$
& \gain $4.81{\pm0.1}$ & \gain $3.86{\pm0.1}$
& \gain $4.48{\pm0.1}$ & \gain $4.00{\pm0.1}$
& \gain $4.10{\pm0.1}$ & $3.77{\pm0.1}$
& $4.22{\pm0.1}$ & \gain $4.50{\pm0.1}$
& \gain $4.36{\pm0.1}$ & \gain $4.80{\pm0.1}$
& \gain \textcolor{bestgreen}{\boldmath$4.85{\pm0.1}$} & $3.80{\pm0.1}$ \\

\midrule
\multicolumn{17}{@{}l}{
\textbf{\textcolor{sectionslate}{\textit{Repeatability}}}
} \\

$\mathrm{Rel}_{\geq4}$ $\uparrow$
& $9{\pm0.7}$ & $10{\pm0.7}$ & $13{\pm0.8}$ & $14{\pm0.8}$
& \gain \textcolor{bestgreen}{\boldmath$19{\pm0.7}$} & $13{\pm0.8}$
& \gain $15{\pm0.7}$ & \gain $15{\pm0.8}$
& \gain $16{\pm0.8}$ & \gain $14{\pm0.8}$
& $14{\pm0.7}$ & $13{\pm0.8}$
& \gain $15{\pm0.7}$ & \gain $16{\pm0.8}$
& \gain $15{\pm0.7}$ & \gain $15{\pm0.8}$ \\

\rowcolor{rowgray}
$\mathrm{Rel}_{\geq6}$ $\uparrow$
& $4{\pm0.5}$ & $8{\pm0.6}$ & $9{\pm0.6}$ & $9{\pm0.7}$
& \gain $10{\pm0.6}$ & $7{\pm0.6}$
& \gain $11{\pm0.6}$ & $8{\pm0.6}$
& $9{\pm0.6}$ & $9{\pm0.7}$
& $9{\pm0.6}$ & $9{\pm0.7}$
& \gain $11{\pm0.6}$ & \gain $12{\pm0.7}$
& \gain \textcolor{bestgreen}{\boldmath$14{\pm0.6}$} & \gain $11{\pm0.7}$ \\

$\mathrm{Rel}_{\geq8}$ $\uparrow$
& $2{\pm0.4}$ & $5{\pm0.5}$ & $6{\pm0.5}$ & $4{\pm0.5}$
& $5{\pm0.5}$ & $4{\pm0.5}$
& $6{\pm0.5}$ & $3{\pm0.4}$
& $6{\pm0.5}$ & $5{\pm0.5}$
& \gain $5{\pm0.5}$ & $4{\pm0.5}$
& \gain $6{\pm0.5}$ & \gain $6{\pm0.5}$
& \gain \textcolor{bestgreen}{\boldmath$7{\pm0.5}$} & $4{\pm0.5}$ \\

\rowcolor{rowgray}
Fragile ($1 \le s_i \le 2$) $\downarrow$
& $15{\pm0.8}$ & $11{\pm0.7}$ & $12{\pm0.8}$ & $9{\pm0.7}$
& \gain \textcolor{bestgreen}{\boldmath$6{\pm0.6}$} & $13{\pm0.8}$
& \gain $9{\pm0.7}$ & \gain $9{\pm0.7}$
& \gain $10{\pm0.7}$ & $14{\pm0.8}$
& $10{\pm0.7}$ & $9{\pm0.7}$
& $9{\pm0.7}$ & \gain $8{\pm0.6}$
& $9{\pm0.7}$ & $13{\pm0.8}$ \\

\bottomrule
\end{tabular}%
}
\end{table}

\paragraph{Does high-likelihood selection improve long-horizon learnability?}

\Figureref{fig:temporal_difficulty_selection} provides direct evidence that high-likelihood selection improves long-horizon learnability. At the same selection ratio, Top-likelihood models consistently achieve higher local likelihood than their Bottom-likelihood counterparts (Top-25\% $>$ Bottom-25\%, Top-50\% $>$ Bottom-50\%, and Top-75\% $>$ Bottom-75\%). Notably, Top-25\% even surpasses full 12K training around the target horizon. These results suggest that warm-start likelihood identifies continuations that are more compatible with the learned short-horizon policy and therefore easier to model during long-horizon refinement.

\paragraph{How much long-horizon data should be retained?} \Tableref{tab:all_selective_results} and \Figureref{fig:temporal_difficulty_selection} show a horizon-dependent quality--quantity trade-off. At 12K, Top-25 gives the most successes ($125{\pm2.2}$), while Top-50 yields slightly stronger $\mathrm{Rel}_{\geq6}$ ($11$ vs.\ $10$); at 16K, Top-75 is best ($126{\pm2.1}$ successes, $\mathrm{Rel}_{\geq6}=14$). This matches \Theoremref{thm:bias_complexity_bound}: smaller $\rho$ improves continuation compatibility but reduces $N_{\mathrm{sel}}$, increasing the estimation burden $O\!\left(\sqrt{\mathcal{C}(H)/N_{\mathrm{sel}}}\right)$. Since $\mathcal{C}(H)$ also grows with horizon, longer horizons require retaining more data to control estimation error.

\begin{wrapfigure}{r}{0.55\linewidth}
    \vspace{-1.5em}
    \centering
    \caption{\textbf{Temporal Supervision Difficulty $\ell^{\mathrm{loc}}_{H}(\cdot;4)$ of the base, selective 12K-refined, and full 12K models.}
    }
    \includegraphics[width=\linewidth]{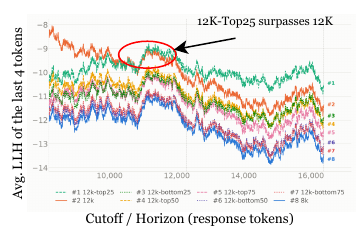}
    \vspace{-3em}
    \label{fig:temporal_difficulty_selection}
\end{wrapfigure}

\paragraph{Does likelihood-based selection matter?}
Likelihood-based selection matters, but its effect depends on the selection ratio and target horizon (\Tableref{tab:all_selective_results}). At 12K with $\rho=50\%$, Top-50 clearly outperforms Bottom-50, with $121{\pm2.3}$ vs.\ $112{\pm2.6}$ successful attempts, $\mathrm{Rel}_{\geq6}=11{\pm0.6}$ vs.\ $8{\pm0.6}$, and $\mathrm{Rel}_{\geq8}=6{\pm0.5}$ vs.\ $3{\pm0.4}$. These gaps exceed the estimated run-to-run variation, whereas at other settings the differences narrow or reverse. Thus, likelihood is a useful compatibility signal, but should be considered jointly with the retained data and target horizon.

\paragraph{Is a longer refinement horizon always beneficial?}
No. At $\rho=50\%$, 12K-Top and 16K-Top are statistically comparable ($121{\pm2.3}$ vs.\ $122{\pm2.4}$ successes; both $\mathrm{Rel}_{\geq6}=11{\pm0.6}$). At $\rho=75\%$, however, extending to 16K improves $119{\pm2.5}\!\rightarrow\!126{\pm2.1}$ successes and $\mathrm{Rel}_{\geq6}$ from $9{\pm0.6}$ to $14{\pm0.6}$. Thus, longer refinement helps only when enough continuation data are retained.

\paragraph{Additional ablations.} We provide other ablations, covering task-level and selection analyses, schedule/data-quantity controls, and cross-benchmark results on Terminal-Bench v2.0 and OpenThoughts-TBLite in \Appendixref{sec:selective_additional}.
\section{Conclusion}

We study the supervision horizon in terminal-agent training and identify an emergence-and-saturation pattern: reliability improves with longer supervision but plateaus beyond an intermediate horizon. Based on this, we propose selective long-horizon refinement, using continuation likelihood to retain more learnable long-horizon trajectories while reducing unnecessary training burden.

\clearpage

\section*{AI Use Statement}
Generative AI tools were used to assist with language polishing, LaTeX formatting, and code debugging. All research ideas, experimental design, analyses, results, and conclusions were developed and verified by the authors.

\section*{Ethics Statement}
This work studies the training and evaluation of terminal agents using existing datasets and benchmarks. It does not involve human subjects or personally identifiable information. While more capable terminal agents could potentially be misused for harmful automation, our work focuses on improving training reliability and efficiency rather than expanding access to harmful capabilities.

\clearpage
\bibliography{iclr2027_conference}
\bibliographystyle{iclr2027_conference}

\clearpage
\appendix
\etocdepthtag.toc{mtappendix}
\renewcommand{\thesection}{\Alph{section}}
\renewcommand{\thesubsection}{\thesection.\arabic{subsection}}
\etocdepthtag.toc{mtappendix}

\begingroup
\hypersetup{linkcolor=black}
\etocsettagdepth{mtbody}{none}
\etocsettagdepth{mtappendix}{subsection}

\etocsettocstyle{%
  \begin{center}
    {\large\scshape Appendix Contents}
  \end{center}
  \vspace{-2pt}\hrule height 0.4pt\vspace{8pt}%
}{%
  \vspace{6pt}\hrule height 0.4pt\vspace{12pt}%
}

\etocsetstyle{section}{}{}{%
  \par\addvspace{6pt}\noindent
  \makebox[1.8em][l]{\bfseries\etocnumber}%
  {\bfseries\etocname}%
  \nobreak\leaders\hbox to 0.7em{\hss.\hss}\hfill
  \nobreak\makebox[1.6em][r]{\bfseries\etocpage}\par
}{}

\etocsetstyle{subsection}{}{}{%
  \par\addvspace{1pt}\noindent\hspace{1.8em}%
  \makebox[2.6em][l]{\etocnumber}%
  \etocname
  \nobreak\leaders\hbox to 0.7em{\hss.\hss}\hfill
  \nobreak\makebox[1.6em][r]{\etocpage}\par
}{}

\tableofcontents
\endgroup

\providecommand{\rhohi}{75\%}
\providecolor{gaingreen}{HTML}{E4F3E7}
\providecolor{bestgreen}{HTML}{1B7A3A}
\providecolor{grpbg}{HTML}{F2F4F7}

\section{Setting}

This section provides the detailed setup used throughout our paper, including the training configuration, evaluation protocol, reliability metrics, and auxiliary analyses underlying the main results.



\subsection{Training Setting}
\label{sec:training_setting_appendix}

For the experiments in \Sectionref{sec:reliability}, \Sectionref{sec:horizon_vs_behaviors} and \Sectionref{sec:method}, we train all models with full-parameter supervised fine-tuning using LLaMA-Factory on the Nemotron terminal-agent corpus, which contains 59,407 ShareGPT-format trajectories. The base model for one-stage training is \texttt{Qwen/Qwen3-8B}. Unless otherwise specified, all runs use the same hyperparameters: two training epochs, BF16 precision, DeepSpeed ZeRO-3 on four NVIDIA H100 GPUs, per-device batch size 1 with no gradient accumulation (global batch size 4), AdamW with learning rate $1\times10^{-5}$, cosine decay without warmup, gradient clipping at 1.0, and no sequence packing. We train complete-data models with supervision cutoffs of 4,096, 8,192, 12,288, and 16,384 tokens, respectively, while keeping the underlying 59,407 trajectories fixed so that these runs isolate the effect of supervision horizon. For selective refinement, we warm-start from the complete 8K checkpoint and continue training for two additional epochs at either 12K or 16K on likelihood-filtered subsets of long trajectories. Candidate trajectories are those extending beyond 8,192 tokens and are ranked by the average teacher-forced log-probability assigned by the 8K model to assistant tokens in the newly exposed continuation, i.e., 8K--12K or 8K--16K. We retain either the highest- or lowest-likelihood 25\%, 50\%, or 75\% of the scored trajectories, yielding approximately 9.8K, 19.6K, and 29.4K training examples, respectively. Trajectories not extending beyond 8K are excluded from the selective-refinement stage. As shown in \Tableref{tab:training_cost}, the 12K Top-50\% setting contains slightly fewer samples than the 16K Top-50\% setting because some candidate trajectories do not contain enough tokens in the required 8K--12K continuation and are therefore skipped during scoring and selection.

\subsection{Evaluation Protocol and Metrics}
\label{sec:eval_protocol}

We evaluate all models on Terminal-Bench v1.0 using the Terminus-2 scaffold, with eight independent attempts per task. For each task $i$, $s_i \in \{0,\dots,8\}$ denotes the number of successful attempts; the reliability metrics $\mathrm{Rel}_{\geq k}$, $\bar{s}_{+}$, and fragile tasks ($1 \le s_i \le 2$) are defined in \Sectionref{sec:reliability}. Because each task has only eight attempts, per-task differences of one or two successes are within sampling noise; in task-level analyses we therefore interpret only large or consistent shifts.

\subsection{Empirical Measures of the Bias--Complexity Trade-off}
\label{sec:empirical_measures}

The metrics in \Equationref{eq:temporal_metrics} are empirical diagnostics rather than estimates of the quantities in \Theoremref{thm:bias_complexity_bound}. The bias $B(H)$ is defined through population risks over the unknown prefix-optimal set $\Pi_H^\star$ and the full-horizon optimum $\pi_T^\star$, while the complexity term depends on the policy class. In particular, $\widehat{B}(H)$ measures the gap to the trained $16$K model, itself a finite-sample estimate, rather than to $\pi_T^\star$, and corresponds to a single element of $\Pi_H^\star$ rather than the supremum in $B(H)$. We therefore construct observable metrics that capture the same intuition. Specifically, $\widehat{B}(H)$ measures the full-horizon modeling gap between the empirically trained horizon-$H$ policy and the 16K reference policy, mirroring the population risk gap in $B(H)$. Meanwhile, $\ell^{\mathrm{loc}}_{H}(K;w)$ measures the average log-likelihood of a local token window on the training trajectories. Tracking this quantity as $K$ increases reveals how the empirical learnability of behavior changes deeper into the trajectory: decreasing local likelihood indicates that later-stage behavior is increasingly difficult for the learned policy to model. Thus, these metrics do not numerically instantiate the theoretical terms, but provide empirical diagnostics for the same bias--difficulty trade-off.

\subsection{Proactive vs.\ Reactive Classification}
\label{sec:proactive_reactive}

For the experiment in \Sectionref{sec:horizon_behavior}, we classify each trial according to whether the agent encounters an execution error during interaction. A trial is \emph{proactive} if no genuine command error appears in any terminal observation returned to the agent, and \emph{reactive} if at least one such error occurs and the agent must respond to it. We inspect the terminal observation after every interaction step, excluding the initial task prompt, and match lines against common error signatures, including Python tracebacks and \texttt{*Error}/\texttt{*Exception}, \texttt{command not found}, \texttt{No such file or directory}, permission and access failures, \texttt{fatal}/\texttt{error}, segmentation faults, missing modules, syntax errors, linker failures, \texttt{FAILED}, \texttt{Killed}, connection failures, and \texttt{apt} \texttt{E:} messages. To avoid false positives, we ignore lines that merely echo the agent's own commands or code, harness/parser messages about response formatting, and warnings. This classification is independent of task outcome: a resolved trial may therefore be either proactive, meaning it succeeds without encountering an error, or reactive, meaning it succeeds after encountering and recovering from an error. The procedure is intentionally mechanical and identifies whether an error occurred, rather than inferring how the agent reasoned about it.

\section{Task-Level Emergence and Saturation of Reliability}
\label{sec:task_level_reliability}

\definecolor{bestgreen}{HTML}{2E7D32}
\definecolor{rulegray}{HTML}{9AA0A6}
\definecolor{headerdeep}{HTML}{EAF2F8}
\definecolor{headerblue}{HTML}{F4F8FB}
\definecolor{sectionslate}{HTML}{34495E}
\definecolor{sectionteal}{HTML}{167D7F}
\definecolor{rowgray}{HTML}{F5F5F5}

\newcommand{\zerodot}{%
  \textcolor{rulegray}{\ensuremath{\cdot}}%
}

\newcommand{\heatlegend}{%
  \begingroup
  \setlength{\fboxsep}{1pt}%
  \colorbox{bestgreen!8}{\strut 1}%
  \hspace{1pt}%
  \colorbox{bestgreen!24}{\strut 3}%
  \hspace{1pt}%
  \colorbox{bestgreen!40}{\strut 5}%
  \hspace{1pt}%
  \colorbox{bestgreen!64}{\strut 8}%
  \endgroup
}

\renewcommand{\hm}[1]{%
  \ifcase#1
    \textcolor{rulegray}{$\cdot$}%
  \or \cellcolor{bestgreen!8}1%
  \or \cellcolor{bestgreen!16}2%
  \or \cellcolor{bestgreen!24}3%
  \or \cellcolor{bestgreen!32}4%
  \or \cellcolor{bestgreen!40}5%
  \or \cellcolor{bestgreen!48}6%
  \or \cellcolor{bestgreen!56}7%
  \or \cellcolor{bestgreen!64}8%
  \else #1%
  \fi
}

\begin{table*}[t]
\centering
\scriptsize
\setlength{\tabcolsep}{2.5pt}
\setlength{\aboverulesep}{0pt}
\setlength{\belowrulesep}{0pt}
\renewcommand{\arraystretch}{0.92}
\arrayrulecolor{rulegray}

\caption{\textbf{\textcolor{sectionslate}{Task-level reliability across one-stage and selective training.}} Each cell reports the number of successful attempts out of eight; shading scales with success count (\heatlegend), and \zerodot{} denotes zero. \textbf{\textcolor{sectionteal}{T}}/\textbf{\textcolor{sectionteal}{B}} denote selection of the highest-/lowest-likelihood continuations during refinement. \textbf{\#Runs} counts the configurations (out of 12) that solve each task at least once; \textbf{\textcolor{bestgreen}{green}} marks the best value in each summary row.}
\label{tab:task_level_reliability}

\begin{tabular}{
@{}
>{\raggedright\arraybackslash}p{0.27\textwidth}
*{12}{c}
c
@{}
}
\toprule

\rowcolor{headerdeep}
&
\multicolumn{4}{c}{\textbf{\textcolor{sectionslate}{One-stage}}}
&
\multicolumn{4}{c}{\textbf{\textcolor{sectionslate}{12K refinement}}}
&
\multicolumn{4}{c}{\textbf{\textcolor{sectionslate}{16K refinement}}}
& \\

\cmidrule(lr){2-5}
\cmidrule(lr){6-9}
\cmidrule(lr){10-13}

\rowcolor{headerblue}
\textbf{Task}
& \textbf{4K}
& \textbf{8K}
& \textbf{12K}
& \textbf{16K}
& \textbf{T50}
& \textbf{B50}
& \textbf{T75}
& \textbf{B75}
& \textbf{T50}
& \textbf{B50}
& \textbf{T75}
& \textbf{B75}
& \textbf{\#Runs} \\

\midrule

\multicolumn{14}{@{}l}{
\textbf{\textcolor{sectionteal}{
\textit{Solved by all 12 configurations}}}} \\

crack-7z-hash.easy
& \hm{4}&\hm{1}&\hm{6}&\hm{4}
& \hm{3}&\hm{6}&\hm{6}&\hm{5}
& \hm{4}&\hm{8}&\hm{6}&\hm{7}
& 12\\

create-bucket
& \hm{1}&\hm{3}&\hm{2}&\hm{3}
& \hm{2}&\hm{3}&\hm{4}&\hm{2}
& \hm{3}&\hm{3}&\hm{1}&\hm{6}
& 12\\

csv-to-parquet
& \hm{8}&\hm{8}&\hm{8}&\hm{7}
& \hm{8}&\hm{7}&\hm{8}&\hm{8}
& \hm{8}&\hm{8}&\hm{8}&\hm{8}
& 12\\

extract-safely
& \hm{5}&\hm{7}&\hm{8}&\hm{8}
& \hm{7}&\hm{7}&\hm{8}&\hm{6}
& \hm{8}&\hm{8}&\hm{8}&\hm{7}
& 12\\

fibonacci-server
& \hm{2}&\hm{2}&\hm{3}&\hm{4}
& \hm{8}&\hm{5}&\hm{5}&\hm{5}
& \hm{7}&\hm{6}&\hm{7}&\hm{6}
& 12\\

fix-pandas-version
& \hm{7}&\hm{7}&\hm{8}&\hm{8}
& \hm{8}&\hm{8}&\hm{8}&\hm{7}
& \hm{8}&\hm{8}&\hm{8}&\hm{8}
& 12\\

fix-permissions
& \hm{7}&\hm{8}&\hm{8}&\hm{8}
& \hm{8}&\hm{8}&\hm{8}&\hm{8}
& \hm{7}&\hm{8}&\hm{8}&\hm{8}
& 12\\

git-workflow-hack
& \hm{2}&\hm{3}&\hm{7}&\hm{3}
& \hm{5}&\hm{4}&\hm{3}&\hm{3}
& \hm{3}&\hm{5}&\hm{7}&\hm{6}
& 12\\

grid-pattern-transform
& \hm{1}&\hm{2}&\hm{1}&\hm{2}
& \hm{5}&\hm{1}&\hm{1}&\hm{3}
& \hm{2}&\hm{2}&\hm{1}&\hm{1}
& 12\\

hello-world
& \hm{8}&\hm{8}&\hm{8}&\hm{8}
& \hm{8}&\hm{8}&\hm{8}&\hm{8}
& \hm{8}&\hm{8}&\hm{8}&\hm{8}
& 12\\

heterogeneous-dates
& \hm{5}&\hm{8}&\hm{4}&\hm{6}
& \hm{8}&\hm{5}&\hm{3}&\hm{4}
& \hm{6}&\hm{6}&\hm{6}&\hm{6}
& 12\\

incompatible-python-fasttext
& \hm{2}&\hm{3}&\hm{2}&\hm{2}
& \hm{1}&\hm{3}&\hm{2}&\hm{6}
& \hm{3}&\hm{5}&\hm{3}&\hm{3}
& 12\\

incompatible-python-fasttext.base\_with\_hint
& \hm{2}&\hm{4}&\hm{8}&\hm{6}
& \hm{6}&\hm{3}&\hm{6}&\hm{8}
& \hm{8}&\hm{7}&\hm{8}&\hm{5}
& 12\\

modernize-fortran-build
& \hm{2}&\hm{3}&\hm{3}&\hm{7}
& \hm{5}&\hm{5}&\hm{5}&\hm{2}
& \hm{4}&\hm{6}&\hm{7}&\hm{1}
& 12\\

new-encrypt-command
& \hm{2}&\hm{2}&\hm{5}&\hm{4}
& \hm{5}&\hm{6}&\hm{6}&\hm{6}
& \hm{6}&\hm{6}&\hm{7}&\hm{5}
& 12\\

simple-web-scraper
& \hm{2}&\hm{8}&\hm{3}&\hm{5}
& \hm{6}&\hm{5}&\hm{5}&\hm{8}
& \hm{7}&\hm{4}&\hm{7}&\hm{7}
& 12\\

sqlite-with-gcov
& \hm{1}&\hm{2}&\hm{4}&\hm{2}
& \hm{6}&\hm{4}&\hm{5}&\hm{4}
& \hm{5}&\hm{5}&\hm{3}&\hm{4}
& 12\\

swe-bench-langcodes
& \hm{5}&\hm{6}&\hm{7}&\hm{6}
& \hm{7}&\hm{7}&\hm{8}&\hm{5}
& \hm{8}&\hm{7}&\hm{8}&\hm{4}
& 12\\

\midrule

\multicolumn{14}{@{}l}{
\textbf{\textcolor{sectionslate}{
\textit{Solved by a subset of configurations}}}} \\

configure-git-webserver
& \hm{0}&\hm{1}&\hm{1}&\hm{2}
& \hm{2}&\hm{1}&\hm{1}&\hm{2}
& \hm{2}&\hm{0}&\hm{2}&\hm{1}
& 10\\

crack-7z-hash
& \hm{0}&\hm{0}&\hm{4}&\hm{5}
& \hm{3}&\hm{4}&\hm{4}&\hm{3}
& \hm{3}&\hm{1}&\hm{4}&\hm{3}
& 10\\

vim-terminal-task
& \hm{4}&\hm{5}&\hm{1}&\hm{3}
& \hm{1}&\hm{1}&\hm{4}&\hm{1}
& \hm{0}&\hm{0}&\hm{1}&\hm{1}
& 10\\

conda-env-conflict-resolution
& \hm{0}&\hm{0}&\hm{1}&\hm{1}
& \hm{1}&\hm{1}&\hm{2}&\hm{1}
& \hm{1}&\hm{0}&\hm{2}&\hm{1}
& 9\\

blind-maze-explorer-algorithm
& \hm{0}&\hm{1}&\hm{0}&\hm{0}
& \hm{1}&\hm{2}&\hm{0}&\hm{2}
& \hm{1}&\hm{1}&\hm{0}&\hm{1}
& 7\\

processing-pipeline
& \hm{0}&\hm{1}&\hm{1}&\hm{0}
& \hm{2}&\hm{3}&\hm{3}&\hm{2}
& \hm{1}&\hm{0}&\hm{0}&\hm{0}
& 7\\

blind-maze-explorer-algorithm.hard
& \hm{1}&\hm{2}&\hm{0}&\hm{1}
& \hm{0}&\hm{0}&\hm{0}&\hm{1}
& \hm{2}&\hm{0}&\hm{0}&\hm{1}
& 6\\

get-bitcoin-nodes
& \hm{2}&\hm{0}&\hm{1}&\hm{1}
& \hm{0}&\hm{1}&\hm{1}&\hm{1}
& \hm{0}&\hm{0}&\hm{0}&\hm{0}
& 6\\

tmux-advanced-workflow
& \hm{2}&\hm{0}&\hm{1}&\hm{0}
& \hm{0}&\hm{0}&\hm{1}&\hm{1}
& \hm{0}&\hm{1}&\hm{2}&\hm{0}
& 6\\

blind-maze-explorer-5x5
& \hm{0}&\hm{0}&\hm{0}&\hm{2}
& \hm{3}&\hm{1}&\hm{0}&\hm{1}
& \hm{0}&\hm{0}&\hm{0}&\hm{1}
& 5\\

oom
& \hm{2}&\hm{0}&\hm{1}&\hm{0}
& \hm{0}&\hm{2}&\hm{0}&\hm{2}
& \hm{0}&\hm{0}&\hm{1}&\hm{0}
& 5\\

simple-sheets-put
& \hm{0}&\hm{0}&\hm{0}&\hm{0}
& \hm{0}&\hm{0}&\hm{1}&\hm{1}
& \hm{1}&\hm{2}&\hm{2}&\hm{0}
& 5\\

git-multibranch
& \hm{1}&\hm{1}&\hm{0}&\hm{0}
& \hm{0}&\hm{0}&\hm{0}&\hm{0}
& \hm{1}&\hm{0}&\hm{1}&\hm{0}
& 4\\

openssl-selfsigned-cert
& \hm{0}&\hm{0}&\hm{3}&\hm{2}
& \hm{0}&\hm{0}&\hm{0}&\hm{1}
& \hm{4}&\hm{0}&\hm{0}&\hm{0}
& 4\\

fix-git
& \hm{0}&\hm{0}&\hm{1}&\hm{0}
& \hm{0}&\hm{1}&\hm{0}&\hm{0}
& \hm{0}&\hm{0}&\hm{0}&\hm{1}
& 3\\

password-recovery
& \hm{0}&\hm{1}&\hm{0}&\hm{0}
& \hm{0}&\hm{0}&\hm{0}&\hm{0}
& \hm{0}&\hm{2}&\hm{0}&\hm{1}
& 3\\

blind-maze-explorer-algorithm.easy
& \hm{0}&\hm{0}&\hm{1}&\hm{0}
& \hm{1}&\hm{0}&\hm{0}&\hm{0}
& \hm{0}&\hm{0}&\hm{0}&\hm{0}
& 2\\

prove-plus-comm
& \hm{0}&\hm{0}&\hm{0}&\hm{0}
& \hm{0}&\hm{0}&\hm{1}&\hm{0}
& \hm{1}&\hm{0}&\hm{0}&\hm{0}
& 2\\

crack-7z-hash.hard
& \hm{0}&\hm{0}&\hm{0}&\hm{0}
& \hm{0}&\hm{0}&\hm{0}&\hm{0}
& \hm{0}&\hm{1}&\hm{0}&\hm{0}
& 1\\

decommissioning-service-with-sensitive-data
& \hm{0}&\hm{0}&\hm{0}&\hm{0}
& \hm{0}&\hm{0}&\hm{0}&\hm{0}
& \hm{0}&\hm{0}&\hm{0}&\hm{1}
& 1\\

eval-mteb
& \hm{0}&\hm{0}&\hm{0}&\hm{0}
& \hm{0}&\hm{0}&\hm{0}&\hm{0}
& \hm{0}&\hm{0}&\hm{0}&\hm{1}
& 1\\

eval-mteb.hard
& \hm{0}&\hm{0}&\hm{0}&\hm{0}
& \hm{0}&\hm{0}&\hm{0}&\hm{0}
& \hm{0}&\hm{0}&\hm{0}&\hm{1}
& 1\\

organization-json-generator
& \hm{0}&\hm{0}&\hm{0}&\hm{0}
& \hm{0}&\hm{0}&\hm{1}&\hm{0}
& \hm{0}&\hm{0}&\hm{0}&\hm{0}
& 1\\

path-tracing
& \hm{0}&\hm{0}&\hm{0}&\hm{0}
& \hm{0}&\hm{0}&\hm{0}&\hm{0}
& \hm{0}&\hm{2}&\hm{0}&\hm{0}
& 1\\

path-tracing-reverse
& \hm{0}&\hm{0}&\hm{0}&\hm{0}
& \hm{1}&\hm{0}&\hm{0}&\hm{0}
& \hm{0}&\hm{0}&\hm{0}&\hm{0}
& 1\\

swe-bench-fsspec
& \hm{0}&\hm{0}&\hm{0}&\hm{0}
& \hm{0}&\hm{0}&\hm{1}&\hm{0}
& \hm{0}&\hm{0}&\hm{0}&\hm{0}
& 1\\

\midrule

\rowcolor{rowgray}
\textbf{Correct tasks ($\mathrm{Rel}_{\geq1}$)}
& \textbf{24}
& \textbf{25}
& \textbf{29}
& \textbf{26}
& \textbf{27}
& \textbf{28}
& \textbf{29}
& \best{31}
& \textbf{28}
& \textbf{25}
& \textbf{26}
& \textbf{30}
& \\

\rowcolor{rowgray}
\textbf{Successful attempts}
& 78
& 97
& 111
& 110
& 121
& 112
& 119
& 117
& 122
& 120
& \best{126}
& 114
& \\

\bottomrule
\end{tabular}
\end{table*}

\begin{table}[t]
\centering
\small
\setlength{\tabcolsep}{6pt}
\setlength{\aboverulesep}{0pt}
\setlength{\belowrulesep}{0pt}
\renewcommand{\arraystretch}{1.15}
\arrayrulecolor{rulegray}
\caption{
\textbf{\textcolor{sectionslate}{Recovery persistence across supervision horizons.}}
\emph{Taxonomy}: successful trajectories over the 36 manually analyzed tasks, classified as
proactive (pitfall anticipated) or reactive (error observed, then repaired).
\emph{Reactive tasks}: successes out of eight on tasks that require recovery; every
\texttt{modernize-fortran-build} success is reactive.
\emph{Pass rate}: over the five tasks whose eight trials are all classified.
\textbf{\textcolor{bestgreen}{Green}}: best in row.
}
\label{tab:persistence}
\begin{tabular}{@{}lcccc@{}}
\toprule
\rowcolor{headerdeep}
\textbf{Metric} & \textbf{4K} & \textbf{8K} & \textbf{12K} & \textbf{16K} \\
\midrule
\multicolumn{5}{@{}l}{\textbf{\textcolor{sectionteal}{\textit{Taxonomy (36 tasks)}}}} \\
Proactive successes & 42 & 53 & \best{56} & 54 \\
\rowcolor{rowgray}
Reactive successes  & 27 & 38 & 41 & \best{42} \\
Reactive share      & 39\% & 42\% & 42\% & \best{44\%} \\
\midrule
\multicolumn{5}{@{}l}{\textbf{\textcolor{sectionslate}{\textit{Reactive tasks}}}} \\
{\footnotesize\texttt{modernize-fortran-build}} & 2/8 & 3/8 & 3/8 & \best{7/8} \\
\rowcolor{rowgray}
{\footnotesize\texttt{fibonacci-server}}        & 2/8 & 2/8 & 3/8 & \best{4/8} \\
{\footnotesize\texttt{simple-web-scraper}} (reactive only) & 0/8 & 2/8 & 0/8 & \best{4/8} \\
\midrule
\rowcolor{rowgray}
\textbf{Pass rate} (5 fully classified tasks) & 50\% & 70\% & 63\% & \best{80\%} \\
\bottomrule
\end{tabular}
\end{table}

Aggregate metrics in \Sectionref{sec:reliability} summarize how reliability changes with the
supervision horizon but not which tasks drive the change.
\Tableref{tab:task_level_reliability} reports per-task successes for all twelve configurations; this section uses the four one-stage columns, and \Appendixref{sec:selective_additional} discusses the selective ones. With eight attempts per task, differences of one or two successes are within sampling noise (\Appendixref{sec:eval_protocol}), so we focus on large or consistent shifts.

\subsection{From Fragile to Repeatable Success}
\label{sec:task_level_emergence}

At 4K, success is sparse and fragile: 15 of the 24 solved tasks succeed in only one or two of eight attempts. Extending the horizon to 12K mostly moves tasks up the reliability distribution rather than adding isolated successes. For example, \texttt{incompatible-python-fasttext.base\_with\_hint} improves from $2/8$ to $8/8$, \texttt{git-workflow-hack} from $2/8$ to $7/8$, \texttt{new-encrypt-command} from $2/8$ to $5/8$, and \texttt{sqlite-with-gcov} from $1/8$ to $4/8$.

Some tasks become solvable only once the horizon is long enough. \texttt{crack-7z-hash} is never solved at 4K or 8K but reaches $4/8$ at 12K and $5/8$ at 16K; \texttt{conda-env-conflict-resolution} is first solved at 12K; and \texttt{blind-maze-explorer-5x5} is solved only at 16K. Successful runs on these tasks take 20--33 episodes of sustained trial and adjustment (\Tableref{tab:horizon_behavior_tasks}), so short-horizon training rarely observes the teacher behavior they require. Consistently, the 16K model produces 21 successful trials lasting at least 20 episodes, spread across 12 tasks, whereas the 8K model produces 11 across 6 tasks.

\subsection{Saturation as Cancellation, Not a Plateau}
\label{sec:task_level_saturation}

From 12K to 16K, the aggregate success count barely moves ($111 \rightarrow 110$), but this hides
substantial task-level movement. Of the 34 tasks solved by at least one one-stage model,
11 improve, 13 worsen, and 10 are unchanged. The improving tasks gain 18 successes and the
degrading tasks lose 19:
\begin{equation}
    \underbrace{+18}_{\text{gained successes}}
    +
    \underbrace{(-19)}_{\text{lost successes}}
    =
    -1.
\end{equation}
The gains concentrate on tasks that require long recovery, such as
\texttt{modernize-fortran-build} ($3/8 \rightarrow 7/8$) and
\texttt{blind-maze-explorer-5x5} ($0 \rightarrow 2/8$).
The losses fall mainly on tasks that peak at 12K:
\texttt{git-workflow-hack} ($7/8 \rightarrow 3/8$),
\texttt{incompatible-python-fasttext.base\_with\_hint} ($8/8 \rightarrow 6/8$),
\texttt{sqlite-with-gcov} ($4/8 \rightarrow 2/8$), and
\texttt{openssl-selfsigned-cert} ($3/8 \rightarrow 2/8$).
In addition, five tasks solved once at 12K become unsolved at 16K.

This cancellation explains the mixed aggregate profile in \Tableref{tab:horizon_reliability}.
Coverage ($\mathrm{Rel}_{\geq1}$) and perfect repeatability ($\mathrm{Rel}_{\geq8}$) fall because
marginal and 12K-peak tasks lose ground. Meanwhile, $\mathrm{Rel}_{\geq4}$, $\bar{s}_{+}$, and the
number of fragile tasks improve because other tasks become more consistent.
Saturation is therefore not a uniform plateau: beyond 12K, additional supervision redistributes
reliability across tasks rather than increasing it.

\paragraph{Caveat.}
A few tasks are solved most often at the shortest horizon; for example, \texttt{oom} and
\texttt{tmux-advanced-workflow} succeed in $2/8$ attempts at 4K and in at most $1/8$ at longer
horizons. These differences are within sampling noise, and we do not interpret them as horizon
effects.

\section{Horizon and Agent Behavior}
\label{sec:persistence}

\newcommand{\ep}[1]{\,\textcolor{rulegray}{\tiny(#1)}}

\begin{table}[t]
\centering
\scriptsize
\setlength{\tabcolsep}{5pt}
\setlength{\aboverulesep}{0pt}
\setlength{\belowrulesep}{0pt}
\renewcommand{\arraystretch}{1.12}
\arrayrulecolor{rulegray}
\caption{
\textbf{Task-level behavior across supervision horizons.}
Successes out of eight, with mean episodes per run in parentheses (shown for solved settings).
Counts are a subset of \Tableref{tab:task_level_reliability};
\textbf{\textcolor{bestgreen}{green}}: best in row.
}
\label{tab:horizon_behavior_tasks}
\begin{tabular}{@{}lcccc@{}}
\toprule
\rowcolor{headerdeep}
\textbf{Task} & 4K & 8K & 12K & 16K \\
\midrule
\multicolumn{5}{@{}l}{\textbf{\textcolor{sectionteal}{\textit{Long-horizon emergence}}}} \\
crack-7z-hash                 & 0 & 0 & 4\ep{27.1} & \best{5}\ep{32.6} \\
conda-env-conflict-resolution & 0 & 0 & \best{1}\ep{32.5} & \best{1}\ep{25.1} \\
blind-maze-explorer-5x5       & 0 & 0 & 0 & \best{2}\ep{20.6} \\
\midrule
\multicolumn{5}{@{}l}{\textbf{\textcolor{sectionteal}{\textit{Efficient recovery}}}} \\
modernize-fortran-build & 2\ep{10.5} & 3\ep{15.5} & 3\ep{11.6} & \best{7}\ep{12.9} \\
\midrule
\multicolumn{5}{@{}l}{\textbf{\textcolor{sectionteal}{\textit{Over-persistence}}}} \\
simple-web-scraper & 2\ep{6.9}  & \best{8}\ep{7.4}  & 3\ep{7.3}  & 5\ep{9.3}  \\
vim-terminal-task  & 4\ep{28.3} & \best{5}\ep{13.0} & 1\ep{18.8} & 3\ep{23.1} \\
\bottomrule
\end{tabular}
\end{table}

\begin{table}[t]
\centering
\scriptsize
\setlength{\tabcolsep}{6pt}
\setlength{\aboverulesep}{0pt}
\setlength{\belowrulesep}{0pt}
\renewcommand{\arraystretch}{1.12}
\arrayrulecolor{rulegray}
\caption{
\textbf{Recovery persistence across supervision horizons.}
Successful trajectories over the 36 manually analyzed tasks, classified as proactive
(pitfall anticipated) or reactive (error observed, then repaired).
Pass rate is computed over the five tasks whose eight trials are all classified.
\textbf{\textcolor{bestgreen}{green}}: best in row.
}
\label{tab:horizon_behavior_share}
\begin{tabular}{@{}lcccc@{}}
\toprule
\rowcolor{headerdeep}
 & 4K & 8K & 12K & 16K \\
\midrule
Proactive successes & 42 & 53 & \best{56} & 54 \\
Reactive successes  & 27 & 38 & 41 & \best{42} \\
Reactive share      & 39\% & 42\% & 42\% & \best{44\%} \\
Pass rate (5 tasks) & 50\% & 70\% & 63\% & \best{80\%} \\
\bottomrule
\end{tabular}
\end{table}

This section complements \Sectionref{sec:horizon_behavior} with task-level evidence on how the
horizon shapes \emph{persistence}: the agent's tendency to keep acting after an error, diagnose
it, and recover rather than terminate.
\Tableref{tab:proactive_reactive} in the main text uses the automatic error-signature classifier
of \Appendixref{sec:proactive_reactive} over all runs. Here we additionally use a finer manual
taxonomy over 36 tasks, which labels a successful trajectory \emph{proactive} if the relevant
pitfall is anticipated before it manifests and \emph{reactive} if the error first occurs and is
then diagnosed and repaired (\Tableref{tab:horizon_behavior_share}).

\subsection{Longer Supervision Enables Recovery}
\label{sec:persistence_recovery}

Reactive successes grow steadily with the horizon, from 27 at 4K to 42 at 16K, and their share of
successful trajectories rises from 39\% to 44\% (\Tableref{tab:horizon_behavior_share}).
Because some tasks are structurally proactive or reactive, this ratio alone is not evidence of
persistence; the more informative comparison is across checkpoints on the \emph{same} task.

\texttt{modernize-fortran-build} is structurally reactive: the build must fail before it can be
repaired, so every success is a recovery. Successes rise from $2/8$ at 4K to $3/8$ at 8K and 12K,
and then to $7/8$ at 16K. Mean episodes at 16K (12.9) remain below those at 8K (15.5), so the gain
reflects better-targeted recovery rather than more attempts (\Tableref{tab:horizon_behavior_tasks}).
\texttt{fibonacci-server} shows a similar trend, rising from $2/8$ to $4/8$.

The contrast is sharpest on \texttt{simple-web-scraper}. The 4K model never recovers after
observing the target bug: it dismisses the observed behavior as acceptable, edits the downstream
report instead of the corrupted CSV, or spends its remaining budget on an unrelated error.
At 16K, four of the five successful trajectories recover reactively, one after 22 episodes of
revisiting its diagnosis. Across the five tasks whose eight trials are all classified, the pass
rate rises from 50\% at 4K to 80\% at 16K.
Longer supervision therefore exposes the learner to later trajectory stages in which an initial
action has already failed, shifting behavior from
$\text{Act} \rightarrow \text{Observe Failure} \rightarrow \text{Terminate}$ toward
$\text{Act} \rightarrow \text{Observe Failure} \rightarrow \text{Diagnose} \rightarrow \text{Retry}
\rightarrow \text{Recover}$.

\subsection{Beyond the Effective Horizon: Over-Persistence}
\label{sec:persistence_over}

The same tendency to continue becomes costly on tasks that shorter horizons already solve.
On \texttt{simple-web-scraper} and \texttt{vim-terminal-task}, the 8K model is best, and longer
horizons both lower success and lengthen runs. For example, 8K solves \texttt{vim-terminal-task}
in $5/8$ attempts with 13.0 episodes per run, whereas 16K solves it in $3/8$ with 23.1
(\Tableref{tab:horizon_behavior_tasks}).

Failed 16K trajectories on \texttt{simple-web-scraper} reveal the mechanism. The agent often
produces a working script, notices a minor formatting issue, and then either enters repetitive
correction or introduces a new error that breaks the solution; one failure repeats the same
incorrect diagnosis over several episodes. In other runs, the agent verifies only coarse
properties and terminates while semantic defects remain.
The same tendency appears in aggregate: 89\% of successful 8K runs finish within 20 episodes,
compared with 81\% at 16K, and mean episodes of successful runs rise from 10.8 at 12K to 12.7 at
16K (\Figureref{fig:episodes_per_run}).

Together, these results reconcile the two faces of 16K. It shows the strongest recovery
persistence on structurally reactive tasks, but this persistence is not selective: the same
drive to continue helps tasks that require recovery and hurts tasks that were already solved.
Behavior therefore moves from \emph{under-persistence} at 4K, through productive persistence
around 12K, to \emph{over-persistence} at 16K, consistent with \Sectionref{sec:horizon_behavior}.

\section{Additional Results for Selective Refinement}
\label{sec:selective_additional}

This section provides additional analyses supporting the selective long-horizon refinement results in \Sectionref{sec:method} and \Sectionref{sec:experiments} of the main paper. We further examine the effect of selection ratio, likelihood-based filtering, refinement horizon, and evaluation benchmark, complementing the main results and ablations on when and why selective refinement improves long-horizon agent training.

\subsection{Task-Level Results}
\label{sec:selective_task_level}

\paragraph{Selective refinement consistently improves aggregate success.}
The advantage of selective refinement is not confined to a single configuration. All eight
filtered models in \Tableref{tab:task_level_reliability} achieve 112--126 successful attempts,
exceeding the best one-stage result of 111. Uniformly supervising all long-horizon continuations
is therefore unnecessary for obtaining the benefits of long-horizon training.

\paragraph{Selection trades coverage against repeatability.}
Successful attempts and task coverage need not move together. 16K-T75 obtains the most successful
attempts (126) but solves 26 tasks, whereas 16K-B75 obtains fewer (114) while solving 30.
Likelihood-based selection thus redistributes reliability: high-likelihood subsets favor repeated
success on solvable tasks, whereas low-likelihood subsets expand the set of tasks solved.

\subsection{Top- vs.\ Bottom-Likelihood Selection}
\label{sec:selective_top_bottom}

\paragraph{Top-likelihood refinement repairs regressions and shortens runs.}
Top-likelihood selection recovers several tasks that regress under full 16K training, usually with
shorter runs (16K columns of \Tableref{tab:task_level_reliability}).
\texttt{fibonacci-server} improves from $4/8$ under full 16K training to $7/8$ under both T50 and
T75, while mean episodes per run nearly halve ($14.0 \rightarrow 7.9$ under T50).
\texttt{incompatible-python-fasttext.base\_with\_hint} returns to $8/8$, matching its 12K peak,
and \texttt{swe-bench-langcodes} rises from $6/8$ to $8/8$.
\texttt{fix-pandas-version} isolates the efficiency effect: every 16K variant solves it $8/8$,
but T50 needs 9.9 episodes per run versus 17.5 for full 16K training.
In aggregate, T75 yields more completed runs than B75 (209 vs.\ 197), with shorter completed
trajectories (10.7 vs.\ 12.2 mean episodes).

Selection is not uniformly beneficial, however. \texttt{vim-terminal-task}, already an
over-persistence case, degrades further under Top selection ($3/8$ under full 16K training vs.\ $1/8$ under T75 and $0/8$ under T50). \texttt{modernize-fortran-build} is highly sensitive to which continuations are retained: it keeps $7/8$ under T75 but collapses to $1/8$ under B75, suggesting that the recovery segments it depends on can be filtered out entirely.

\paragraph{Bottom-likelihood refinement trades repeatability for breadth.} B75 solves the most tasks of any 16K model ($\mathrm{Rel}_{\geq1}=30$), including \texttt{eval-mteb}, \texttt{eval-mteb.hard}, and \texttt{decommissioning-service-with-sensitive-data}, which no one-stage or Top-selected model solves. B50 similarly adds \texttt{crack-7z-hash.hard} and \texttt{path-tracing}. These new successes are mostly isolated ($1$--$2/8$): relative to T75, B75 has fewer consistently solved tasks ($\mathrm{Rel}_{\geq6}$: 11 vs.\ 14) and more fragile ones (13 vs.\ 9; \Tableref{tab:all_selective_results}). Low-likelihood continuations thus appear to carry rare, diverse behaviors that expand coverage, whereas high-likelihood continuations consolidate repeatable solving, consistent with the capacity- and budget-dependent effects of likelihood-based selection reported by \citet{cuong2026capacitydependent}.

\subsection{How do training schedule, data quantity, and likelihood selection interact?}
To disentangle the effects of the two-stage schedule, data quantity, and likelihood-based selection, we add two controls. First, we refine the 8K warm-start model at 12K and 16K using all eligible trajectories, isolating the effect of the two-stage schedule from filtering. Second, for each retention ratio $\rho\in\{25\%,50\%,75\%\}$, we uniformly sample the same number of eligible trajectories as the corresponding Top/Bottom-$\rho$ setting, so comparisons among Top, Random, and Bottom are made at a fixed data budget. \Tableref{tab:selection_controls} supports the bias--complexity trade-off in \Theoremref{thm:bias_complexity_bound}. First, two-stage training with all eligible data improves over one-stage training from $111\pm2.3$ to $120\pm2.4$ successes at 12K and from $110\pm2.7$ to $123\pm2.4$ at 16K, indicating that the warm-start schedule itself provides a substantial benefit. Second, the Random controls reveal a horizon-dependent data-quantity effect. At 12K, increasing retention from 25\% to 75\% changes success only from $116\pm2.5$ to $118\pm2.6$, which is within the estimated run-to-run variation. At 16K, however, the corresponding increase is much larger, from $111\pm2.8$ to $122\pm2.4$, consistent with the finite-sample term $O\!\left(\sqrt{\mathcal{C}(H)/N_{\mathrm{sel}}}\right)$: as $\mathcal{C}(H)$ grows with horizon, using too few trajectories becomes increasingly costly. Finally, fixed-budget comparisons isolate selection quality. At 12K-25\%, Top reaches $125\pm2.2$ successes, compared with $116\pm2.5$ for Random and $108\pm2.8$ for Bottom, a gap larger than the reported variation. At 16K-75\%, Top similarly reaches $126\pm2.1$, versus $122\pm2.4$ for Random and $114\pm2.7$ for Bottom. In contrast, some settings show little separation, for example, at 16K-50\%, Top, Random, and Bottom obtain $122\pm2.4$, $121\pm2.5$, and $120\pm2.5$ successes, respectively. Overall, the controls support a joint quality--quantity view: the two-stage schedule improves initialization, likelihood selection can provide additional gains when the separation exceeds evaluation variation, and retaining sufficient data becomes increasingly important at longer horizons.

\begin{table}[t]
\centering
\small
\caption{
\textbf{Disentangling two-stage training, data quantity, and likelihood selection.}
All two-stage models warm-start from 8K and refine at the target horizon on a
random, highest- (Top), or lowest-likelihood (Bottom) $\rho\%$ of continuations.
One-stage and Top/Bottom results are measured. \colorbox{gaingreen}{Shaded}: improves over one-stage training at the same target;
\textbf{\textcolor{bestgreen}{green}}: best in column within each target.
}
\label{tab:selection_controls}
\setlength{\tabcolsep}{3.2pt}
\setlength{\aboverulesep}{0pt}
\setlength{\belowrulesep}{0pt}
\renewcommand{\arraystretch}{1.18}
\arrayrulecolor{rulegray}

\begin{tabular}{@{}ll ccccc@{}}
\toprule
\rowcolor{headerdeep}
\textbf{Selection} & $\boldsymbol{\rho}$
& \textbf{Succ.}\,$\uparrow$
& \multicolumn{3}{c}{\textbf{Reliability}\,$\uparrow$}
& \textbf{Fragile}\,$\downarrow$ \\
\cmidrule(lr){4-6}
\rowcolor{headerblue}
& & & $\mathrm{Rel}_{\geq4}$ & $\mathrm{Rel}_{\geq6}$ &
$\mathrm{Rel}_{\geq8}$ & \\
\midrule

\multicolumn{7}{@{}l}{\textbf{\textcolor{sectionteal}{\textit{12K target}}}} \\

One-stage (full) & ---
  & $111{\pm}2.3$
  & $13{\pm}0.8$
  & $9{\pm}0.6$
  & \best{$6{\pm}0.5$}
  & $12{\pm}0.8$ \\

Two-stage All & 100\%
  & \gain $120{\pm}2.4$
  & \gain $16{\pm}0.8$
  & \gain $10{\pm}0.6$
  & \best{$6{\pm}0.5$}
  & \gain $9{\pm}0.7$ \\

\rowcolor{rowgray}
Random &
  & \gain $116{\pm}2.5$
  & \gain $15{\pm}0.8$
  & $8{\pm}0.6$
  & $4{\pm}0.5$
  & \gain $10{\pm}0.7$ \\

\rowcolor{rowgray}
Top &
  & \gain\best{$125{\pm}2.2$}
  & \gain\best{$19{\pm}0.7$}
  & \gain $10{\pm}0.6$
  & $5{\pm}0.5$
  & \gain\best{$6{\pm}0.6$} \\

\rowcolor{rowgray}
Bottom & \multirow{-3}{*}{25\%}
  & $108{\pm}2.8$
  & $13{\pm}0.8$
  & $7{\pm}0.6$
  & $4{\pm}0.5$
  & $13{\pm}0.8$ \\

Random &
  & \gain $117{\pm}2.5$
  & \gain $15{\pm}0.8$
  & $9{\pm}0.6$
  & $5{\pm}0.5$
  & \gain $10{\pm}0.7$ \\

Top &
  & \gain $121{\pm}2.3$
  & \gain $15{\pm}0.7$
  & \gain\best{$11{\pm}0.6$}
  & \best{$6{\pm}0.5$}
  & \gain $9{\pm}0.7$ \\

Bottom & \multirow{-3}{*}{50\%}
  & \gain $112{\pm}2.6$
  & \gain $15{\pm}0.8$
  & $8{\pm}0.6$
  & $3{\pm}0.4$
  & \gain $9{\pm}0.7$ \\

\rowcolor{rowgray}
Random &
  & \gain $118{\pm}2.6$
  & \gain $15{\pm}0.8$
  & $9{\pm}0.7$
  & $5{\pm}0.5$
  & \gain $10{\pm}0.8$ \\

\rowcolor{rowgray}
Top &
  & \gain $119{\pm}2.5$
  & \gain $16{\pm}0.8$
  & $9{\pm}0.6$
  & \best{$6{\pm}0.5$}
  & \gain $10{\pm}0.7$ \\

\rowcolor{rowgray}
Bottom & \multirow{-3}{*}{75\%}
  & \gain $117{\pm}2.7$
  & \gain $14{\pm}0.8$
  & $9{\pm}0.7$
  & $5{\pm}0.5$
  & $14{\pm}0.8$ \\

\midrule

\multicolumn{7}{@{}l}{\textbf{\textcolor{sectionslate}{\textit{16K target}}}} \\

One-stage (full) & ---
  & $110{\pm}2.7$
  & $14{\pm}0.8$
  & $9{\pm}0.7$
  & $4{\pm}0.5$
  & $9{\pm}0.7$ \\

Two-stage All & 100\%
  & \gain $123{\pm}2.4$
  & \gain\best{$16{\pm}0.8$}
  & \gain $12{\pm}0.7$
  & \gain $6{\pm}0.5$
  & \gain\best{$8{\pm}0.7$} \\

\rowcolor{rowgray}
Random &
  & \gain $111{\pm}2.8$
  & $14{\pm}0.8$
  & $9{\pm}0.7$
  & \gain $5{\pm}0.5$
  & $10{\pm}0.7$ \\

\rowcolor{rowgray}
Top &
  & \gain $114{\pm}2.6$
  & $14{\pm}0.7$
  & $9{\pm}0.6$
  & \gain $5{\pm}0.5$
  & $10{\pm}0.7$ \\

\rowcolor{rowgray}
Bottom & \multirow{-3}{*}{25\%}
  & $108{\pm}2.9$
  & $13{\pm}0.8$
  & $9{\pm}0.7$
  & $4{\pm}0.5$
  & $9{\pm}0.7$ \\

Random &
  & \gain $121{\pm}2.5$
  & \gain $15{\pm}0.8$
  & \gain $11{\pm}0.7$
  & \gain $6{\pm}0.5$
  & $9{\pm}0.7$ \\

Top &
  & \gain $122{\pm}2.4$
  & \gain $15{\pm}0.7$
  & \gain $11{\pm}0.6$
  & \gain $6{\pm}0.5$
  & $9{\pm}0.7$ \\

Bottom & \multirow{-3}{*}{50\%}
  & \gain $120{\pm}2.5$
  & \gain\best{$16{\pm}0.8$}
  & \gain $12{\pm}0.7$
  & \gain $6{\pm}0.5$
  & \gain\best{$8{\pm}0.6$} \\

\rowcolor{rowgray}
Random &
  & \gain $122{\pm}2.4$
  & \gain\best{$16{\pm}0.8$}
  & \gain $12{\pm}0.7$
  & \gain $6{\pm}0.5$
  & $9{\pm}0.8$ \\

\rowcolor{rowgray}
Top &
  & \gain\best{$126{\pm}2.1$}
  & \gain $15{\pm}0.7$
  & \gain\best{$14{\pm}0.6$}
  & \gain\best{$7{\pm}0.5$}
  & $9{\pm}0.7$ \\

\rowcolor{rowgray}
Bottom & \multirow{-3}{*}{75\%}
  & \gain $114{\pm}2.7$
  & \gain $15{\pm}0.8$
  & \gain $11{\pm}0.7$
  & $4{\pm}0.5$
  & $13{\pm}0.8$ \\

\bottomrule
\end{tabular}
\end{table}

\subsection{Results on Other Benchmarks}
\label{sec:selective_other_benchmarks}

\paragraph{The emergence--saturation pattern generalizes across benchmarks.}
As shown in \Tableref{tab:all_selective_results_tb2} and \Tableref{tab:all_selective_results_tblite}, the one-stage results reproduce the same qualitative trend observed in the main experiments. On Terminal-Bench v2.0, successful attempts increase from $44{\pm2.6}$ at 4K to $65{\pm2.3}$ at 12K, but remain essentially unchanged at 16K ($64{\pm2.6}$); similarly, on OpenThoughts-TBLite they increase from $96{\pm2.9}$ to $139{\pm2.4}$ before saturating at $137{\pm2.7}$. The reliability profile also exhibits the same redistribution beyond the effective horizon: from 12K to 16K, task coverage decreases ($17\!\rightarrow\!15$ on TB v2.0; $36\!\rightarrow\!32$ on TBLite), while mean success among solved tasks increases ($3.82\!\rightarrow\!4.27$; $3.86\!\rightarrow\!4.28$), $\mathrm{Rel}_{\geq4}$ increases, and the number of fragile tasks decreases. Thus, across benchmarks, longer supervision initially broadens task solving, whereas supervision beyond the intermediate horizon primarily makes already-solvable tasks more repeatable rather than increasing aggregate reliability.

\paragraph{Selective refinement consistently improves long-horizon reliability.}
\Tableref{tab:all_selective_results_tb2} and \Tableref{tab:all_selective_results_tblite} further show that the benefit of selective refinement transfers across benchmarks. On Terminal-Bench v2.0, 12K Top-50 improves successful attempts from $65{\pm2.3}$ to $71{\pm2.3}$, while 16K Top-75 improves from $64{\pm2.6}$ to $73{\pm2.1}$. The same pattern is stronger on OpenThoughts-TBLite, where the corresponding improvements are $139{\pm2.4}\!\rightarrow\!151{\pm2.3}$ and $137{\pm2.7}\!\rightarrow\!155{\pm2.2}$. These gains are larger than the estimated run-to-run variation, supporting the main-paper observation that uniformly supervising all long-horizon continuations is unnecessary: selecting a subset of compatible continuations can yield higher reliability than full long-horizon training.

\paragraph{The effect of selection depends on both likelihood and retained data.}
As shown in \Tableref{tab:all_selective_results_tb2} and \Tableref{tab:all_selective_results_tblite}, the cross-benchmark results also preserve the interaction between continuation likelihood and selection ratio observed in the main experiments. At 12K and $\rho=50\%$, Top-likelihood selection consistently exceeds Bottom-likelihood selection in aggregate success ($71$ vs.\ $66$ on TB v2.0 and $151$ vs.\ $141$ on TBLite), while also improving higher-order repeatability. However, the advantage is not uniform across every metric or selection ratio: Bottom-likelihood subsets can retain broader task coverage in some settings, and the strongest 16K result emerges only when sufficient data are retained, with Top-75 achieving $73{\pm2.1}$ successes on TB v2.0 and $155{\pm2.2}$ on TBLite. This is consistent with the main conclusion that likelihood is a useful compatibility signal, but effective refinement must jointly balance continuation learnability and data sufficiency.

\begin{table}[t]
\centering
\caption{
\textbf{Reliability of one-stage and selective two-stage training on Terminal-Bench v2.0.}
Two-stage models warm-start from 8K and refine at 12K or 16K on the highest- (\topl) or lowest-likelihood (\botl) $\rho\%$ of continuations.
Values are reported with estimated run-to-run variation.
\colorbox{gaingreen}{Shaded}: improves over one-stage training at the same horizon;
\textbf{\textcolor{bestgreen}{green}}: best in row.
}
\label{tab:all_selective_results_tb2}
\setlength{\tabcolsep}{3.2pt}
\setlength{\aboverulesep}{0pt}
\setlength{\belowrulesep}{0pt}
\renewcommand{\arraystretch}{1.18}
\arrayrulecolor{rulegray}
\resizebox{\textwidth}{!}{%
\begin{tabular}{@{}l cccc cccccc cccccc@{}}
\toprule
\rowcolor{headerdeep}
& \multicolumn{4}{c}{\textbf{\textcolor{sectionslate}{One-stage}}}
& \multicolumn{6}{c}{\textbf{\textcolor{sectionslate}{12K refinement}} (from 8K)}
& \multicolumn{6}{c}{\textbf{\textcolor{sectionslate}{16K refinement}} (from 8K)} \\
\cmidrule(lr){2-5}\cmidrule(lr){6-11}\cmidrule(lr){12-17}

\rowcolor{headerblue}
& \multicolumn{4}{c}{\textit{training horizon}}
& \multicolumn{2}{c}{$\rho{=}25\%$}
& \multicolumn{2}{c}{$\rho{=}50\%$}
& \multicolumn{2}{c}{$\rho{=}75\%$}
& \multicolumn{2}{c}{$\rho{=}25\%$}
& \multicolumn{2}{c}{$\rho{=}50\%$}
& \multicolumn{2}{c}{$\rho{=}75\%$} \\
\cmidrule(lr){6-7}\cmidrule(lr){8-9}\cmidrule(lr){10-11}
\cmidrule(lr){12-13}\cmidrule(lr){14-15}\cmidrule(lr){16-17}

\rowcolor{headerblue}
\textbf{Metric}
& 4K & 8K & 12K & 16K
& \topl & \botl & \topl & \botl & \topl & \botl
& \topl & \botl & \topl & \botl & \topl & \botl \\

\midrule
\multicolumn{17}{@{}l}{
\textbf{\textcolor{sectionteal}{\textit{Aggregate reliability}}}
} \\

Succ.\ attempts $\uparrow$
& $44{\pm2.6}$ & $55{\pm2.5}$ & $65{\pm2.3}$ & $64{\pm2.6}$
& \gain $72{\pm2.2}$ & $63{\pm2.7}$
& \gain $71{\pm2.3}$ & \gain $66{\pm2.5}$
& \gain $70{\pm2.4}$ & \gain $69{\pm2.6}$
& \gain $66{\pm2.5}$ & $63{\pm2.8}$
& \gain $71{\pm2.3}$ & \gain $70{\pm2.5}$
& \gain\best{$73{\pm2.1}$} & \gain $66{\pm2.6}$ \\

\rowcolor{rowgray}
$\mathrm{Rel}_{\geq1}$ $\uparrow$
& $13{\pm0.7}$ & $14{\pm0.7}$ & $17{\pm0.6}$ & $15{\pm0.7}$
& $15{\pm0.7}$ & $16{\pm0.8}$
& $16{\pm0.7}$ & $16{\pm0.8}$
& $17{\pm0.7}$ & \gain\best{$18{\pm0.8}$}
& \gain $16{\pm0.7}$ & $14{\pm0.8}$
& \gain $16{\pm0.7}$ & $15{\pm0.8}$
& $15{\pm0.7}$ & \gain $17{\pm0.8}$ \\

$\bar{s}_{+}$ $\uparrow$
& $3.38{\pm0.1}$ & $3.93{\pm0.1}$ & $3.82{\pm0.1}$ & $4.27{\pm0.1}$
& \gain $4.80{\pm0.1}$ & \gain $3.94{\pm0.1}$
& \gain $4.44{\pm0.1}$ & \gain $4.13{\pm0.1}$
& \gain $4.12{\pm0.1}$ & \gain $3.83{\pm0.1}$
& $4.13{\pm0.1}$ & \gain $4.50{\pm0.1}$
& \gain $4.44{\pm0.1}$ & \gain $4.67{\pm0.1}$
& \gain\best{$4.87{\pm0.1}$} & $3.88{\pm0.1}$ \\

\midrule
\multicolumn{17}{@{}l}{
\textbf{\textcolor{sectionslate}{\textit{Repeatability}}}
} \\

$\mathrm{Rel}_{\geq4}$ $\uparrow$
& $5{\pm0.5}$ & $6{\pm0.5}$ & $8{\pm0.6}$ & $9{\pm0.6}$
& \gain\best{$11{\pm0.6}$} & $8{\pm0.6}$
& \gain $9{\pm0.6}$ & \gain $9{\pm0.6}$
& \gain $9{\pm0.6}$ & $8{\pm0.6}$
& $8{\pm0.6}$ & $8{\pm0.6}$
& $9{\pm0.6}$ & $9{\pm0.6}$
& $9{\pm0.6}$ & $9{\pm0.6}$ \\

\rowcolor{rowgray}
$\mathrm{Rel}_{\geq6}$ $\uparrow$
& $2{\pm0.4}$ & $4{\pm0.5}$ & $5{\pm0.5}$ & $5{\pm0.5}$
& \gain $6{\pm0.5}$ & $4{\pm0.5}$
& \gain $6{\pm0.5}$ & $5{\pm0.5}$
& $5{\pm0.5}$ & $5{\pm0.5}$
& $5{\pm0.5}$ & $5{\pm0.5}$
& \gain $6{\pm0.5}$ & \gain $7{\pm0.6}$
& \gain\best{$8{\pm0.6}$} & \gain $6{\pm0.5}$ \\

$\mathrm{Rel}_{\geq8}$ $\uparrow$
& $1{\pm0.3}$ & $2{\pm0.4}$ & $3{\pm0.4}$ & $2{\pm0.4}$
& $3{\pm0.4}$ & $2{\pm0.4}$
& \gain\best{$4{\pm0.5}$} & $2{\pm0.4}$
& \gain\best{$4{\pm0.5}$} & $3{\pm0.4}$
& \gain $3{\pm0.4}$ & $2{\pm0.4}$
& \gain $3{\pm0.4}$ & \gain $3{\pm0.4}$
& \gain\best{$4{\pm0.5}$} & $2{\pm0.4}$ \\

\rowcolor{rowgray}
Fragile ($1 \le s_i \le 2$) $\downarrow$
& $7{\pm0.6}$ & $6{\pm0.6}$ & $7{\pm0.7}$ & $5{\pm0.6}$
& \gain\best{$4{\pm0.5}$} & $8{\pm0.7}$
& \gain $5{\pm0.6}$ & \gain $5{\pm0.6}$
& \gain $6{\pm0.6}$ & $8{\pm0.7}$
& $6{\pm0.6}$ & $5{\pm0.6}$
& $5{\pm0.6}$ & $5{\pm0.6}$
& $5{\pm0.6}$ & $8{\pm0.7}$ \\

\bottomrule
\end{tabular}%
}
\end{table}

\begin{table}[t]
\centering
\caption{
\textbf{Reliability of one-stage and selective two-stage training on OpenThoughts-TBLite.}
Two-stage models warm-start from 8K and refine at 12K or 16K on the highest- (\topl) or lowest-likelihood (\botl) $\rho\%$ of continuations.
Values are reported with estimated run-to-run variation.
\colorbox{gaingreen}{Shaded}: improves over one-stage training at the same horizon;
\textbf{\textcolor{bestgreen}{green}}: best in row.
}
\label{tab:all_selective_results_tblite}
\setlength{\tabcolsep}{3.2pt}
\setlength{\aboverulesep}{0pt}
\setlength{\belowrulesep}{0pt}
\renewcommand{\arraystretch}{1.18}
\arrayrulecolor{rulegray}
\resizebox{\textwidth}{!}{%
\begin{tabular}{@{}l cccc cccccc cccccc@{}}
\toprule
\rowcolor{headerdeep}
& \multicolumn{4}{c}{\textbf{\textcolor{sectionslate}{One-stage}}}
& \multicolumn{6}{c}{\textbf{\textcolor{sectionslate}{12K refinement}} (from 8K)}
& \multicolumn{6}{c}{\textbf{\textcolor{sectionslate}{16K refinement}} (from 8K)} \\
\cmidrule(lr){2-5}\cmidrule(lr){6-11}\cmidrule(lr){12-17}

\rowcolor{headerblue}
& \multicolumn{4}{c}{\textit{training horizon}}
& \multicolumn{2}{c}{$\rho{=}25\%$}
& \multicolumn{2}{c}{$\rho{=}50\%$}
& \multicolumn{2}{c}{$\rho{=}75\%$}
& \multicolumn{2}{c}{$\rho{=}25\%$}
& \multicolumn{2}{c}{$\rho{=}50\%$}
& \multicolumn{2}{c}{$\rho{=}75\%$} \\
\cmidrule(lr){6-7}\cmidrule(lr){8-9}\cmidrule(lr){10-11}
\cmidrule(lr){12-13}\cmidrule(lr){14-15}\cmidrule(lr){16-17}

\rowcolor{headerblue}
\textbf{Metric}
& 4K & 8K & 12K & 16K
& \topl & \botl & \topl & \botl & \topl & \botl
& \topl & \botl & \topl & \botl & \topl & \botl \\

\midrule
\multicolumn{17}{@{}l}{
\textbf{\textcolor{sectionteal}{\textit{Aggregate reliability}}}
} \\

Succ.\ attempts $\uparrow$
& $96{\pm2.9}$ & $118{\pm2.7}$ & $139{\pm2.4}$ & $137{\pm2.7}$
& \gain $153{\pm2.2}$ & $135{\pm2.8}$
& \gain $151{\pm2.3}$ & \gain $141{\pm2.6}$
& \gain $149{\pm2.5}$ & \gain $147{\pm2.7}$
& \gain $141{\pm2.6}$ & $135{\pm2.9}$
& \gain $151{\pm2.4}$ & \gain $149{\pm2.5}$
& \gain\best{$155{\pm2.2}$} & \gain $141{\pm2.7}$ \\

\rowcolor{rowgray}
$\mathrm{Rel}_{\geq1}$ $\uparrow$
& $28{\pm0.8}$ & $30{\pm0.8}$ & $36{\pm0.7}$ & $32{\pm0.8}$
& $32{\pm0.7}$ & $34{\pm0.8}$
& $34{\pm0.7}$ & $34{\pm0.8}$
& $36{\pm0.7}$ & \gain\best{$38{\pm0.8}$}
& \gain $34{\pm0.7}$ & $30{\pm0.8}$
& \gain $34{\pm0.7}$ & $32{\pm0.8}$
& $32{\pm0.7}$ & \gain $36{\pm0.8}$ \\

$\bar{s}_{+}$ $\uparrow$
& $3.43{\pm0.1}$ & $3.93{\pm0.1}$ & $3.86{\pm0.1}$ & $4.28{\pm0.1}$
& \gain $4.78{\pm0.1}$ & \gain $3.97{\pm0.1}$
& \gain $4.44{\pm0.1}$ & \gain $4.15{\pm0.1}$
& \gain $4.14{\pm0.1}$ & \gain $3.87{\pm0.1}$
& $4.15{\pm0.1}$ & \gain $4.50{\pm0.1}$
& \gain $4.44{\pm0.1}$ & \gain $4.66{\pm0.1}$
& \gain\best{$4.84{\pm0.1}$} & $3.92{\pm0.1}$ \\

\midrule
\multicolumn{17}{@{}l}{
\textbf{\textcolor{sectionslate}{\textit{Repeatability}}}
} \\

$\mathrm{Rel}_{\geq4}$ $\uparrow$
& $11{\pm0.7}$ & $13{\pm0.7}$ & $17{\pm0.8}$ & $19{\pm0.8}$
& \gain\best{$23{\pm0.7}$} & $17{\pm0.8}$
& \gain $19{\pm0.7}$ & \gain $19{\pm0.8}$
& \gain $19{\pm0.8}$ & $17{\pm0.8}$
& $17{\pm0.7}$ & $17{\pm0.8}$
& $19{\pm0.7}$ & $19{\pm0.8}$
& $19{\pm0.7}$ & $19{\pm0.8}$ \\

\rowcolor{rowgray}
$\mathrm{Rel}_{\geq6}$ $\uparrow$
& $4{\pm0.5}$ & $9{\pm0.6}$ & $11{\pm0.6}$ & $11{\pm0.7}$
& \gain $13{\pm0.6}$ & $8{\pm0.6}$
& \gain $13{\pm0.6}$ & $11{\pm0.6}$
& $11{\pm0.6}$ & $11{\pm0.7}$
& $11{\pm0.6}$ & $11{\pm0.7}$
& \gain $13{\pm0.6}$ & \gain $15{\pm0.7}$
& \gain\best{$17{\pm0.6}$} & \gain $13{\pm0.7}$ \\

$\mathrm{Rel}_{\geq8}$ $\uparrow$
& $2{\pm0.4}$ & $4{\pm0.5}$ & $6{\pm0.5}$ & $4{\pm0.5}$
& $6{\pm0.5}$ & $4{\pm0.5}$
& \gain $8{\pm0.5}$ & $4{\pm0.5}$
& \gain $8{\pm0.5}$ & $6{\pm0.5}$
& \gain $6{\pm0.5}$ & $4{\pm0.5}$
& \gain $6{\pm0.5}$ & \gain $6{\pm0.5}$
& \gain\best{$9{\pm0.5}$} & $4{\pm0.5}$ \\

\rowcolor{rowgray}
Fragile ($1 \le s_i \le 2$) $\downarrow$
& $15{\pm0.8}$ & $13{\pm0.7}$ & $15{\pm0.8}$ & $11{\pm0.7}$
& \gain\best{$9{\pm0.6}$} & $17{\pm0.8}$
& \gain $11{\pm0.7}$ & \gain $11{\pm0.7}$
& \gain $13{\pm0.7}$ & $17{\pm0.8}$
& $13{\pm0.7}$ & $11{\pm0.7}$
& $11{\pm0.7}$ & $11{\pm0.7}$
& $11{\pm0.7}$ & $17{\pm0.8}$ \\

\bottomrule
\end{tabular}%
}
\end{table}

\section{Missing Proofs}
\label{appdendix:proofs}

In this section, we provide the detailed derivation of \Theoremref{thm:bias_complexity_bound} (or
\Theoremref{thm:effective_horizon_detailed}) and establish several consequences for
choosing the supervision horizon. In particular, \Corollaryref{cor:shorter_better} characterizes when shorter-horizon supervision yields a tighter bound than full-horizon training, \Corollaryref{cor:sweetspot} and \Corollaryref{cor:intermediate} characterize when an intermediate horizon minimizes a surrogate bound, and \Corollaryref{cor:data_horizon} describes how this horizon changes with the amount of expert data.

Our analysis proceeds in three steps.
First, we formalize the benefit of longer supervision through
\emph{temporal supervision bias}, which captures the full-horizon
performance gap induced by truncating teacher trajectories. Second, we
show that extending the supervision horizon exposes the learner to a
richer set of decision histories, increasing the effective policy
complexity and hence the finite-sample estimation burden. Third, we
connect prefix-level estimation error to full-horizon performance and
combine these two effects to obtain the bias--complexity bound in
\Theoremref{thm:effective_horizon_detailed}.

We then use this result to answer three questions about the bound: when shorter-horizon supervision yields a tighter guarantee than full-horizon training, under what conditions an intermediate horizon minimizes a surrogate of the bound, and how this horizon shifts as the amount of expert data grows. These results characterize the guarantee rather than the true error, consistent with the scope stated below. Finally, we connect the decision-level analysis to the token-level supervision horizon used in our experiments.

\paragraph{Setup.}
We build on the decision-level imitation-learning formulation in \Sectionref{sec:background}. Let $\Pi$ denote the student policy class and $\pi^{\mathrm{T}}$ the teacher policy. Let $H$ denote the number of teacher decisions retained for supervision. We assume that $N$ teacher trajectories are sampled independently from $\pi^{\mathrm{T}}$ and have lengths $T_i\leq T$. For $t>T_i$, we set $\ell\equiv 0$, so that $L_H$ and $\widehat{L}_H$ are well defined for every $H\leq T$. For any $H \leq T$, supervision is restricted to the first $H$ decisions, while the policy continues to condition on the complete history $h_t$ at each step.

\paragraph{Scope of the analysis.}
The analysis concerns imitation risk under teacher-visited histories, for an empirical risk minimizer over a finite policy class with bounded loss. For log-loss, boundedness requires clipping probabilities below at some $\epsilon>0$, giving $M=\ell_{\max}\log(1/\epsilon)$ for actions of at most $\ell_{\max}$ tokens. Neural networks trained for a finite number of optimizer steps on a token-averaged objective (\Equationref{eq:token-average}) only approximate this setting. The resulting bound does not by itself imply that the true error, or task success under the agent's own rollouts, is non-monotone in $H$, and it makes no prediction about trajectory length or persistence. We therefore use it to identify two competing quantities and treat the behavioral findings of \Sectionref{sec:horizon_behavior} as empirical.

\begin{definition}[Horizon-$H$ imitation risk]
Let $\ell\!\left(\pi(\cdot\mid h_t),a_t\right)$ denote an action-level imitation loss satisfying
\begin{equation}
    0 \leq \ell\!\left(\pi(\cdot\mid h_t),a_t\right) \leq M.
\end{equation}
The population and empirical imitation risks under supervision horizon $H$ are
\begin{equation}
L_H(\pi)=\mathbb{E}_{\tau\sim\pi^{\mathrm{T}}}\left[\frac{1}{H}\sum_{t=1}^{H}\ell\!\left(\pi(\cdot\mid h_t),a_t\right)\right],\qquad\widehat{L}_H(\pi)=\frac{1}{N}\sum_{i=1}^{N}\frac{1}{H}\sum_{t=1}^{H}\ell\!\left(\pi(\cdot\mid h_t^{(i)}),a_t^{(i)}\right),
\end{equation}
where $\mathcal{D}=\{\tau^{(i)}\}_{i=1}^{N}$ denotes the $N$ teacher trajectories. Behavior cloning with horizon $H$ learns
\begin{equation}
    \widehat{\pi}_H\in\arg\min_{\pi\in\Pi}\widehat{L}_H(\pi).
\end{equation}
\end{definition}

\begin{definition}[Horizon-induced policy complexity]
\label{def:horizon_policy_complexity}
Let
\begin{equation}
    \mathcal{H}_{\leq H}
    =
    \left\{
    h_t :
    1 \leq t \leq H,\;
    \Pr_{\tau \sim \pi^{\mathrm{T}}}(h_t) > 0
    \right\}
\end{equation}
denote the set of interaction histories reachable under the teacher policy within the first $H$ decisions. Two policies $\pi,\pi'\in\Pi$ are equivalent at horizon $H$, written $\pi\equiv_H\pi'$, if
\begin{equation}
    \pi(\cdot\mid h)=\pi'(\cdot\mid h),
    \qquad
    \forall h\in\mathcal{H}_{\leq H}.
\end{equation}
The corresponding quotient class and horizon-induced policy complexity are
\begin{equation}
    \Pi_H=\Pi/\!\equiv_H,
    \qquad
    \mathcal{C}(H)=\log|\Pi_H|.
\end{equation}
\end{definition}

For infinite policy classes, the union bound in \Lemmaref{lem:estimation} must be replaced by a covering or symmetrization argument, with $\mathcal{C}(H)$ replaced by the corresponding complexity of the loss class restricted to $\mathcal{H}_{\leq H}$. We do not pursue this extension here.

\begin{lemma}[Policy complexity increases with supervision horizon]
\label{lem:complexity}
For any $H_1 < H_2$,
\begin{equation}
\mathcal{C}(H_1) \leq \mathcal{C}(H_2).
\end{equation}
\end{lemma}

\begin{proof}
Since
\begin{equation}
\mathcal{H}_{\leq H_1} \subseteq \mathcal{H}_{\leq H_2},
\end{equation}
any two policies distinguishable within horizon $H_1$ remain distinguishable within horizon $H_2$. Thus, increasing the supervision horizon may split equivalence classes but cannot merge them, implying
\begin{equation}
|\Pi_{H_1}| \leq |\Pi_{H_2}|,
\qquad
\mathcal{C}(H_1) \leq \mathcal{C}(H_2).
\end{equation}
\end{proof}

\Lemmaref{lem:complexity} formalizes that longer supervision exposes the learner to a richer set of distinguishable behaviors. This increase in complexity reflects the diversity of observed decision histories, not an increase in the number of model parameters.

\begin{definition}[Prefix-optimal set and pessimistic reference risk]
\label{def:population_optimal}
Let
\begin{equation}
    L_H^\star := \min_{\pi\in\Pi} L_H(\pi),
    \qquad
    \Pi_H^\star := \arg\min_{\pi\in\Pi} L_H(\pi),
    \qquad
    \pi_T^\star\in\arg\min_{\pi\in\Pi} L_T(\pi),
\end{equation}
and define the pessimistic full-horizon risk of the prefix-optimal set as
\begin{equation}
    R_H := \sup_{\pi\in\Pi_H^\star} L_T(\pi).
\end{equation}
Since $0\leq\ell\leq M$, $R_H\leq M$ is finite.
\end{definition}

Supervision over the first $H$ decisions constrains a policy only on histories in $\mathcal{H}_{\leq H}$ (\Definitionref{def:horizon_policy_complexity}). Policies that agree on these histories have identical prefix risk but may behave differently afterward, so $\Pi_H^\star$ generally contains policies with different full-horizon risks. $R_H$ accounts for all of them rather than for a single tie-broken minimizer.

\begin{definition}[Temporal supervision bias]
\label{def:temporal_bias_appendix}
We define the temporal supervision bias at horizon $H$ (restated from \Definitionref{def:temporal_bias}) as
\begin{equation}
    B(H):=R_H-L_T(\pi_T^\star)
    =\sup_{\pi\in\Pi_H^\star} L_T(\pi)-L_T(\pi_T^\star).
\end{equation}
Since $\Pi_H^\star$ is non-empty and $\pi_T^\star$ minimizes $L_T$ over $\Pi\supseteq\Pi_H^\star$, we have $B(H)\geq 0$. Moreover, every $\pi\in\Pi_T^\star$ (the case $H=T$ of $\Pi_H^\star$) is a minimizer of $L_T$ and therefore attains $L_T(\pi)=L_T(\pi_T^\star)$, so $R_T=L_T(\pi_T^\star)$ and $B(T)=0$. Thus, $B(H)$ measures the worst-case excess full-horizon risk left undetermined by supervising only the first $H$ decisions.
\end{definition}

\begin{assumption}[Monotone temporal supervision bias]
\label{ass:monotone_bias}
For any $H_1 < H_2 \leq T$,
\begin{equation}
    R_{H_2}\leq R_{H_1},
    \qquad\text{equivalently}\qquad
    B(H_2) \leq B(H_1).
\end{equation}
\end{assumption}

\begin{lemma}[Finite-sample estimation error]
\label{lem:estimation}
Assume that the $N$ teacher trajectories are independently sampled and that $0\leq \ell \leq M$. Then, with probability at least $1-\delta$,
\begin{equation}
L_H(\widehat{\pi}_H)-L_H^\star
\leq
M
\sqrt{
\frac{
2\left[\mathcal{C}(H)+\log\frac{2}{\delta}\right]
}{N}
}.
\end{equation}
\end{lemma}

\begin{proof}
Define the trajectory-level prefix loss as
\begin{equation}
Z_H(\pi;\tau)
=
\frac{1}{H}
\sum_{t=1}^{H}
\ell\!\left(\pi(\cdot\mid h_t),a_t\right),
\qquad
0\leq Z_H(\pi;\tau)\leq M.
\end{equation}
Although decisions within a trajectory are dependent, the variables
$Z_H(\pi;\tau^{(1)}),\ldots,Z_H(\pi;\tau^{(N)})$
are independent across teacher trajectories.

For a fixed policy $\pi$, Hoeffding's inequality gives
\begin{equation}
\Pr\!\left(
\left|L_H(\pi)-\widehat{L}_H(\pi)\right|>\epsilon
\right)
\leq
2\exp\!\left(-\frac{2N\epsilon^2}{M^2}\right).
\end{equation}
Both $L_H(\pi)$ and $\widehat{L}_H(\pi)$ depend on $\pi$ only through its behavior on $\mathcal{H}_{\leq H}$, and are therefore constant on each equivalence class of $\Pi_H$. A union bound over the $|\Pi_H|$ classes thus yields
\begin{equation}
\Pr\!\left(
\sup_{\pi\in\Pi}
\left|L_H(\pi)-\widehat{L}_H(\pi)\right|>\epsilon
\right)
\leq
2|\Pi_H|
\exp\!\left(-\frac{2N\epsilon^2}{M^2}\right).
\end{equation}
Setting the right-hand side to $\delta$ and using
$\mathcal{C}(H)=\log|\Pi_H|$ gives
\begin{equation}
\sup_{\pi\in\Pi}
\left|L_H(\pi)-\widehat{L}_H(\pi)\right|
\leq
M
\sqrt{
\frac{\mathcal{C}(H)+\log(2/\delta)}{2N}
}.
\end{equation}
Fix any $\pi'\in\Pi_H^\star$. Since $\widehat{\pi}_H$ minimizes the empirical risk, $\widehat{L}_H(\widehat{\pi}_H)\leq\widehat{L}_H(\pi')$, and hence
\begin{equation}
\begin{aligned}
L_H(\widehat{\pi}_H)-L_H^\star
&=
\bigl[L_H(\widehat{\pi}_H)-\widehat{L}_H(\widehat{\pi}_H)\bigr]
+
\bigl[\widehat{L}_H(\widehat{\pi}_H)-\widehat{L}_H(\pi')\bigr]
+
\bigl[\widehat{L}_H(\pi')-L_H(\pi')\bigr]\\
&\leq
2
\sup_{\pi\in\Pi}
\left|L_H(\pi)-\widehat{L}_H(\pi)\right|.
\end{aligned}
\end{equation}
Substituting the uniform bound above proves the result.
\end{proof}

Training controls performance only over the supervised prefix, whereas our goal is performance over the full task horizon. Note that the full-horizon risk $L_T$, unlike $L_H$, is \emph{not} constant on the equivalence classes of $\Pi_H$: prefix-equivalent policies may differ after step $H$. The transfer coefficient below is therefore defined relative to the pessimistic reference $R_H$, which covers every prefix-optimal policy.

\begin{definition}[Prefix-to-full-horizon transfer coefficient]
\label{def:prefix-to-full-horizon-coefficient}
For supervision horizon $H$, define (restated from \eqref{eq:transfer_coefficient})
\begin{equation}
\Gamma_H
:=
\sup_{\pi\in\Pi:\;L_H(\pi)>L_H^\star}
\frac{
\left[L_T(\pi)-R_H\right]_+
}{
L_H(\pi)-L_H^\star
},
\qquad
[x]_+ := \max\{x,0\},
\end{equation}
with $\Gamma_H:=0$ if no policy has $L_H(\pi)>L_H^\star$.
\end{definition}

The coefficient $\Gamma_H$ measures how much prefix-level excess risk can amplify when evaluated over the full horizon. We assume there exists a constant $\Gamma$ such that
\begin{equation}
\Gamma_H \leq \Gamma,
\qquad
\forall H\leq T.
\end{equation}

For a finite policy class, each $\Gamma_H$ is a maximum over finitely many ratios and is therefore finite, but it can be arbitrarily large, and $\Gamma$ enters the bound multiplicatively through $C_1$. The assumption is therefore substantive as a quantitative condition. Because supervision constrains a policy only on $\mathcal{H}_{\leq H}$, a policy with small prefix excess risk may still behave very differently after step $H$, which inflates $\Gamma_H$. A moderate $\Gamma$ thus requires the policy class to couple late-stage behavior to early-stage behavior, for example through parameter sharing across histories. Defining $\Gamma_H$ relative to $R_H$, rather than to a single minimizer, ensures that the transfer inequality also holds for policies that tie on the prefix (\Lemmaref{lem:transfer}).

\begin{lemma}[Prefix-to-full-horizon transfer]
\label{lem:transfer}
For every $\pi\in\Pi$,
\begin{equation}
L_T(\pi)-R_H
\leq
\Gamma_H
\left[L_H(\pi)-L_H^\star\right].
\label{eq:transfer_all}
\end{equation}
Consequently, under $\Gamma_H\leq\Gamma$, with probability at least $1-\delta$,
\begin{equation}
L_T(\widehat{\pi}_H)-R_H
\leq
\Gamma M
\sqrt{
\frac{
2\left[\mathcal{C}(H)+\log\frac{2}{\delta}\right]
}{N}
}.
\end{equation}
\end{lemma}

\begin{proof}
Since $L_H(\pi)\geq L_H^\star$ for every $\pi\in\Pi$, two cases are exhaustive.

\emph{Case 1: $L_H(\pi)=L_H^\star$.} Then $\pi\in\Pi_H^\star$, so $L_T(\pi)\leq R_H$ by definition of $R_H$. The left-hand side of \eqref{eq:transfer_all} is non-positive and the right-hand side is zero, so the inequality holds.

\emph{Case 2: $L_H(\pi)>L_H^\star$.} By definition of $\Gamma_H$,
\begin{equation}
L_T(\pi)-R_H
\leq
\left[L_T(\pi)-R_H\right]_+
\leq
\Gamma_H\left[L_H(\pi)-L_H^\star\right].
\end{equation}

Applying \eqref{eq:transfer_all} to $\pi=\widehat{\pi}_H$, then $\Gamma_H\leq\Gamma$ and \Lemmaref{lem:estimation}, proves the second claim.
\end{proof}

\begin{lemma}[Error decomposition]
\label{lem:decomposition}
For any supervision horizon $H$,
\begin{equation}
L_T(\widehat{\pi}_H)-L_T(\pi_T^\star)
=
B(H)+\mathcal{E}_N(H),
\qquad
\mathcal{E}_N(H)
:=
L_T(\widehat{\pi}_H)-R_H.
\end{equation}
\end{lemma}

\begin{proof}
Adding and subtracting $R_H$ gives
\begin{equation}
L_T(\widehat{\pi}_H)-L_T(\pi_T^\star)
=
\left[L_T(\widehat{\pi}_H)-R_H\right]
+
\left[R_H-L_T(\pi_T^\star)\right].
\end{equation}
The two terms are $\mathcal{E}_N(H)$ and $B(H)$, respectively. Note that $\mathcal{E}_N(H)$ may be negative, e.g., when $\widehat{\pi}_H$ is a prefix-optimal policy with better late-stage behavior than the worst one in $\Pi_H^\star$.
\end{proof}

\begin{theorem}[Detailed Bias--Complexity Bound]
\label{thm:effective_horizon_detailed}
Consider the history-conditioned terminal-agent imitation-learning problem above. Assume that:
\begin{enumerate}
    \item the $N$ teacher trajectories are sampled independently;
    \item the imitation loss is bounded, $0\leq \ell \leq M$;
    \item the policy class $\Pi$ is finite, so that minimizers exist and $\mathcal{C}(H)=\log|\Pi_H|<\infty$;
    \item the transfer coefficient satisfies $\Gamma_H\leq\Gamma<\infty$.
\end{enumerate}
Then, for every $1\leq H\leq T$, with probability at least $1-\delta$,
\begin{equation}
L_T(\widehat{\pi}_H)-L_T(\pi_T^\star)
\leq
B(H)
+
C_1
\sqrt{
\frac{\mathcal{C}(H)+\log(2/\delta)}{N}
},
\qquad
C_1=\sqrt{2}\,\Gamma M.
\end{equation}
\end{theorem}

\begin{proof}
By \Lemmaref{lem:decomposition} and \Lemmaref{lem:transfer},
\begin{equation}
\begin{aligned}
L_T(\widehat{\pi}_H)-L_T(\pi_T^\star)
&=
B(H)+\mathcal{E}_N(H) \\
&\leq
B(H)
+
\Gamma M
\sqrt{
\frac{
2\left[\mathcal{C}(H)+\log(2/\delta)\right]
}{N}
}.
\end{aligned}
\end{equation}
Substituting $C_1=\sqrt{2}\,\Gamma M$ proves the result.
\end{proof}

For the remainder of this section, we denote the estimation term of the bound by
\begin{equation}
\label{eq:estimation_term}
\varepsilon_N(H) := C_1\sqrt{\frac{\mathcal{C}(H)+\log(2/\delta)}{N}},
\end{equation}
which upper-bounds the estimation error $\mathcal{E}_N(H)$ of \Lemmaref{lem:decomposition} with probability at least $1-\delta$. By \Lemmaref{lem:complexity}, $\varepsilon_N(H)$ is non-decreasing in $H$. Note that \Theoremref{thm:effective_horizon_detailed} does not require \Assumptionref{ass:monotone_bias}; monotonicity is used only to interpret the bound as a trade-off.

\textbf{Interpretation.}
\Theoremref{thm:effective_horizon_detailed} identifies two competing effects of increasing the supervision horizon. The first term, $B(H)$, is the \emph{temporal supervision bias}, the worst-case full-horizon risk left undetermined by truncating supervision at $H$. Under \Assumptionref{ass:monotone_bias}, it is non-increasing in $H$, reflecting the additional information provided by later-stage teacher behavior. The second term, $\varepsilon_N(H)$, is the \emph{finite-sample estimation term}. By \Lemmaref{lem:complexity}, it is non-decreasing in $H$, since longer horizons expose the learner to a richer set of behaviorally distinguishable histories.

The bound therefore reflects a bias--complexity trade-off: additional supervision tightens the guarantee when the reduction in temporal supervision bias outweighs the increase in the estimation term. This is consistent with the empirical observation that performance improves with horizon initially but shows diminishing returns beyond an intermediate regime, although the bound alone does not establish that the true error follows this pattern.

\begin{definition}[Bound-optimal horizon]
\label{def:effective_horizon}
The bound-optimal horizon is
\begin{equation}
H_N^\star\in\arg\min_{1\leq H\leq T}\bigl\{B(H)+\varepsilon_N(H)\bigr\}.
\end{equation}
It is the horizon at which the guarantee of \Theoremref{thm:effective_horizon_detailed} is tightest, and it need not coincide with the empirically best horizon.
\end{definition}

\begin{corollary}[Full-horizon supervision]
\label{cor:full_horizon}
At $H=T$, the temporal supervision bias vanishes and the bound becomes
\begin{equation}
B(T)=0,
\qquad
L_T(\widehat{\pi}_T)-L_T(\pi_T^\star)
\leq
C_1
\sqrt{
\frac{\mathcal{C}(T)+\log(2/\delta)}{N}
}.
\end{equation}
Moreover, by \Lemmaref{lem:complexity},
\begin{equation}
\mathcal{C}(T)
=
\max_{H\leq T}\mathcal{C}(H).
\end{equation}
Thus, full-horizon supervision removes temporal supervision bias but incurs the largest horizon-dependent estimation term.
\end{corollary}

\begin{corollary}[When shorter supervision is preferable]
\label{cor:shorter_better}
For any $H<T$, the horizon-$H$ bound is tighter than the full-horizon bound if and only if
\begin{equation}
B(H)<\varepsilon_N(T)-\varepsilon_N(H),
\end{equation}
where $\varepsilon_N$ is the estimation term in \eqref{eq:estimation_term}. In particular, since $B(H)\geq 0$, this requires $\mathcal{C}(H)<\mathcal{C}(T)$.
\end{corollary}

\begin{proof}
By \Theoremref{thm:effective_horizon_detailed} and \Corollaryref{cor:full_horizon}, the horizon-$H$ and full-horizon bounds are
\begin{equation}
B(H)+\varepsilon_N(H),
\qquad
B(T)+\varepsilon_N(T)=\varepsilon_N(T),
\end{equation}
respectively. Hence, horizon $H$ gives a tighter bound if and only if
\begin{equation}
B(H)+\varepsilon_N(H)<\varepsilon_N(T),
\end{equation}
which is equivalent to the stated condition. The second claim follows because $\varepsilon_N$ is strictly increasing in $\mathcal{C}(H)$, so $\varepsilon_N(T)-\varepsilon_N(H)>0$ requires $\mathcal{C}(H)<\mathcal{C}(T)$.
\end{proof}

To characterize diminishing returns, we impose an additional assumption on the temporal benefit contributed by later stages of a trajectory.

\begin{assumption}[Diminishing temporal novelty]
\label{ass:diminishing}
There exists a non-increasing function $\beta(H)$ such that
\begin{equation}
B(H)\leq\beta(H),
\qquad
\Delta_\beta(H):=\beta(H)-\beta(H+1),
\qquad
\Delta_\beta(H+1)\leq\Delta_\beta(H).
\end{equation}
\end{assumption}

A useful special case is geometric decay, $\beta(H)=A\rho^H$ with $A>0$ and $0<\rho<1$, for which $\Delta_\beta(H)=A(1-\rho)\rho^H$. Under \Assumptionref{ass:diminishing}, the bound of \Theoremref{thm:effective_horizon_detailed} is itself bounded by the surrogate
\begin{equation}
U_N(H):=\beta(H)+\varepsilon_N(H),
\qquad
\Delta_\varepsilon(H):=\varepsilon_N(H+1)-\varepsilon_N(H)\geq 0.
\end{equation}

\begin{corollary}[Marginal characterization of the sweet spot]
\label{cor:sweetspot}
For every $H<T$,
\begin{equation}
U_N(H+1)<U_N(H)
\iff
\Delta_\beta(H)>\Delta_\varepsilon(H).
\end{equation}
Thus, extending the supervision horizon tightens the surrogate bound only while the marginal reduction in temporal supervision bias exceeds the marginal increase in the estimation term.
\end{corollary}

\begin{proof}
By definition, $U_N(H+1)-U_N(H)=-\Delta_\beta(H)+\Delta_\varepsilon(H)$.
\end{proof}

\begin{assumption}[Single crossing]
\label{ass:single_crossing}
There exists $H_0\in\{1,\ldots,T\}$ such that $\Delta_\beta(H)>\Delta_\varepsilon(H)$ for all $H<H_0$, and $\Delta_\beta(H)\leq\Delta_\varepsilon(H)$ for all $H_0\leq H<T$.
\end{assumption}

\begin{corollary}[Intermediate surrogate-optimal horizon]
\label{cor:intermediate}
Under \Assumptionref{ass:single_crossing}, $H_0\in\arg\min_{1\leq H\leq T}U_N(H)$. Moreover, $H_0<T$ if and only if $\Delta_\beta(T-1)\leq\Delta_\varepsilon(T-1)$.
\end{corollary}

\begin{proof}
By \Corollaryref{cor:sweetspot}, $U_N$ is strictly decreasing on $\{1,\ldots,H_0\}$ and non-decreasing on $\{H_0,\ldots,T\}$, so $H_0$ is a minimizer. The second claim follows from the definition of $H_0$.
\end{proof}

\begin{corollary}[Effect of additional expert data]
\label{cor:data_horizon}
For a fixed horizon $H$, $\varepsilon_N(H)=O\bigl(\sqrt{\mathcal{C}(H)/N}\bigr)\to 0$ as $N\to\infty$. Moreover, let $H_0(N)$ denote the crossing point of \Assumptionref{ass:single_crossing} at sample size $N$. If the assumption holds at both $N$ and $N'>N$, then $H_0(N')\geq H_0(N)$.
\end{corollary}

\begin{proof}
The first claim is immediate from \eqref{eq:estimation_term}. For the second, write $c_\delta=\log(2/\delta)$ and $\Delta_\varepsilon^{(N)}$ for $\Delta_\varepsilon$ at sample size $N$. Then
\begin{equation}
\Delta_\varepsilon^{(N)}(H)=\frac{C_1}{\sqrt{N}}\left[\sqrt{\mathcal{C}(H+1)+c_\delta}-\sqrt{\mathcal{C}(H)+c_\delta}\right],
\end{equation}
where the bracket is non-negative by \Lemmaref{lem:complexity} and independent of $N$. Hence $\Delta_\varepsilon^{(N)}(H)$ is non-increasing in $N$, while $\Delta_\beta(H)$ does not depend on $N$. It follows that for every $H<H_0(N)$,
\begin{equation}
\Delta_\beta(H)>\Delta_\varepsilon^{(N)}(H)\geq\Delta_\varepsilon^{(N')}(H).
\end{equation}
Suppose, for contradiction, that $H_0(N')<H_0(N)$. Applying \Assumptionref{ass:single_crossing} at $N'$ with $H=H_0(N')$ gives $\Delta_\beta(H_0(N'))\leq\Delta_\varepsilon^{(N')}(H_0(N'))$, which contradicts the display above since $H_0(N')<H_0(N)$. Therefore $H_0(N')\geq H_0(N)$.
\end{proof}

Thus, more expert data weakens the statistical penalty of longer horizons and moves the surrogate-optimal horizon weakly later.

\paragraph{Connection to token-level supervision.}
The analysis above is expressed in terms of a decision horizon $H$, while our implementation controls a token budget $K$. As defined in \Sectionref{sec:sft}, $K$ induces the trajectory-dependent decision horizon
\begin{equation}
H_K(\tau)=\max\{t:n_t(\tau)\leq K\},
\qquad
K_1<K_2\Longrightarrow H_{K_1}(\tau)\leq H_{K_2}(\tau),
\end{equation}
where $n_t(\tau)$ is the cumulative number of tokens through action $a_t$. Thus, increasing $K$ monotonically extends the supervised portion of every trajectory. Two differences from the analysis remain. First, $H_K(\tau)$ varies across trajectories, whereas the analysis uses a common $H$. Second, an action that straddles the budget is partially supervised, which $H_K(\tau)$ ignores. The analysis should therefore be read as describing the qualitative effect of extending supervision, not as a quantitative model of token-level truncation.






\end{document}